\documentclass[journal]{IEEEtran}
\IEEEoverridecommandlockouts

\usepackage{algpseudocode}
\usepackage{amsmath}
\usepackage{indentfirst}
\usepackage[linesnumbered,ruled]{algorithm2e}

\usepackage{hyperref}
\usepackage{cite}
\usepackage{amsmath,amssymb,amsfonts}

\usepackage{epsfig}

\usepackage{graphicx}
\usepackage{textcomp}
\usepackage{xcolor}
\usepackage{epstopdf}
\usepackage{subfigure}
\usepackage{booktabs}
\usepackage{amsthm}
\usepackage{enumerate}

\usepackage{xcolor}

\newtheorem{definition}{Definition}

\newtheorem{lemma}{Lemma}
\newtheorem{theorem}{Theorem}
\newtheorem{Proposition}{Proposition}

\def\BibTeX{{\rm B\kern-.05em{\sc i\kern-.025em b}\kern-.08em
    T\kern-.1667em\lower.7ex\hbox{E}\kern-.125emX}}
\begin{document}
 \title{Tiny Multi-Agent DRL for Twins Migration in UAV Metaverses: A Multi-Leader Multi-Follower Stackelberg Game Approach}

\author{Jiawen Kang, Yue Zhong, Minrui Xu, Jiangtian Nie, Jinbo Wen, Hongyang Du,  Dongdong Ye, \\Xumin Huang, Dusit Niyato, \textit{Fellow, IEEE},  Shengli Xie, \textit{Fellow, IEEE}
\thanks{
This work was supported in part by NSFC under grant No. 62102099 and  U22A2054, and the Pearl River Talent Recruitment Program under Grant 2021QN02S643, and Guangzhou Basic Research Program under Grant 2023A04J1699, in part by Guangzhou Basic Research Program under Grant 2023A04J0340, in part by Energy Research Test-Bed and Industry Partnership Funding Initiative, Energy Grid (EG) 2.0 programme, DesCartes and MOE Tier 1 (RG87/22). (\textit{*Corresponding author: Xumin Huang})

Jiawen Kang is with the School of Automation at Guangdong University of Technology (GDUT) and also 111 Center for Intelligent Batch Manufacturing based on IoT Technology, Guangzhou 510006, China (e-mail: kavinkang@gdut.edu.cn). Yue Zhong is with the School of Automation at GDUT and also Key Laboratory of Intelligent Information Processing and System Integration of IoT, Ministry of Education, Guangzhou 510006, China (e-mail: 3220001516@mail2.gdut.edu.cn). Dongdong Ye is with the School of Automation at GDUT and also Guangdong-HongKong-Macao Joint Laboratory for Smart Discrete Manufacturing, Guangzhou 510006, China (e-mail: dongdongye8@163.com). Xumin Huang is with the School of Automation at GDUT and also Key Laboratory of Intelligent Detection and IoT in Manufacturing, Ministry of Education, Guangzhou 510006, China (e-mail: huangxu\_min@163.com). Shengli Xie is with the School of Automation at GDUT and also Guangdong Key Laboratory of IoT Information Technology, Guangzhou 510006, China (e-mail: shlxie@gdut.edu.cn).


Minrui Xu, Jiangtian Nie, Hongyang Du, and Dusit Niyato are with the School of Computer Science and Engineering, Nanyang Technological University, Singapore, Singapore (e-mail: minrui001@e.ntu.edu.sg; jnie001@e.ntu.edu.sg; hongyang001@e.ntu.edu.sg; DNIYATO@ntu.edu.sg).


Jinbo Wen is with the Computer Science and Technology, Nanjing University of Aeronautics and Astronautics, Nanjing, China (e-mail: jinbo1608@163.com).
}
}

\maketitle

\begin{abstract}
The synergy between Unmanned Aerial Vehicles (UAVs) and metaverses is giving rise to an emerging paradigm named UAV metaverses, which create a unified ecosystem that blends physical and virtual spaces, transforming drone interaction and virtual exploration. UAV Twins (UTs), as the digital twins of UAVs that revolutionize UAV applications by making them more immersive, realistic, and informative, are deployed and updated on ground base stations, e.g., RoadSide Units (RSUs), to offer metaverse services for UAV Metaverse Users (UMUs). Due to the dynamic mobility of UAVs and limited communication coverages of RSUs, it is essential to perform real-time UT migration to ensure seamless immersive experiences for UMUs. However, selecting appropriate RSUs and optimizing the required bandwidth is challenging for achieving reliable and efficient UT migration. To address the challenges, we propose a tiny machine learning-based Stackelberg game framework based on pruning techniques for efficient UT migration in UAV metaverses. Specifically, we formulate a multi-leader multi-follower Stackelberg model considering a new immersion metric of UMUs in the utilities of UAVs. Then, we design a Tiny Multi-Agent Deep Reinforcement Learning (Tiny MADRL) algorithm to obtain the tiny networks representing the optimal game solution. Specifically, the actor-critic network leverages the pruning techniques to reduce the number of network parameters and achieve model size and computation reduction, allowing for efficient implementation of Tiny MADRL.  {Numerical results demonstrate that our proposed schemes have better performance than traditional schemes.}
\end{abstract}


\section{Introduction}
Unmanned Aerial Vehicles (UAVs) have garnered significant prominence in contemporary times owing to their multifaceted utility in expanding the scope of mobile communications, enhancing surveillance capabilities, and contributing to the development of intelligent urban environments \cite{hu2022digital}. The expeditious advancement of state-of-the-art technologies, encompassing blockchain, Artificial Intelligence (AI), and eXtended Reality (XR), coupled with the ubiquity of mobile devices, portends the emergence of metaverses as a transformative paradigm for the forthcoming generation \cite{kangblockchain}. This evolution is anticipated to facilitate interpersonal engagements and interactions with virtual entities within virtual environments \cite{luoprivacy}.
UAV metaverses are considered hybrid immersive realms that combine XR technologies with real-time sensing data collected by UAVs. The metaverses offer customized and intelligent services, such as smart farming and UAV-virtual games, to the users who engage with them, known as UAV Metaverse Users (UMUs). 
 {UAV Twin (UT) plays a crucial role in UAV metaverses, serving as an essential component that encompasses precise and comprehensive digital replicas of UAVs \cite{Jinbo}.} The UTs are instrumental in ensuring synchronization and real-time updates for UMUs, bridging the gap between physical and virtual spaces \cite{luoprivacy}. Consequently, users can seamlessly access up-to-date information and immerse themselves in captivating experiences within UAV metaverses.

 {Executing high-fidelity and low-latency rendering tasks for diverse UTs necessitates substantial allocations of communication, computing, and storage resources \cite{chen,wang2022task}. UAVs may not efficiently build high-fidelity virtual models, the computationally intensive tasks need to be offloaded to ground base stations, e.g., RoadSide Units (RSUs) equipped with edge servers, that have adequate computing and bandwidth resources\cite{zhang2023learning, duresource}.} Each RSU can maintain multiple UTs to provide UMUs with multiple metaverse services. Owing to the flexibility and dynamic mobility of UAVs\cite{wen2023freshness}, the distance between the UAVs and the RSUs in which their UTs are currently deployed is likely to become more and more remote. 
Moreover, the constrained communication coverages of RSUs pose a challenge in ensuring the uninterrupted provision of metaverse services by UAVs to UMUs through various UTs \cite{chen}.  {Therefore, it is crucial to achieve the seamless migration of UTs from current RSUs to other RSUs to ensure uninterrupted and immersive service experiences for UMUs.
However, to obtain a reliable and efficient UT migration, it is challenging to select suitable RSUs and optimize the required bandwidth. Specifically, the RSUs are responsible for providing bandwidth resources for UT migration and autonomously determining their bandwidth prices. Subsequently, UAVs select the target RSUs for UT migration and determine the bandwidth demand by considering the pricing strategies established by the RSUs. }


 {Tiny Machine Learning (TML) is dedicated to crafting and advancing machine learning techniques tailored for execution on embedded systems and Internet of Things (IoT) devices \cite{Disabato}. The methods harness approximation and pruning strategies to efficiently diminish the computational load and memory demands inherent in machine learning algorithms \cite{Disabato}. To address the above challenges, we propose a tiny learning-based game approach framework based on the pruning techniques for efficient UT migration in UAV metaverses. We formulate a multi-leader multi-follower Stackelberg game model to ensure efficient resource allocation between RSUs and UAVs within UT migration.
In Stage I, RSUs acting as the leaders establish the bandwidth selling prices. The prices are determined by considering the expected budgets of UAVs and the pricing strategies of other RSUs. 
During Stage II, the follower UAVs assess the required bandwidth and make purchasing decisions from various RSUs based on the pricing strategies established in Stage I.} To summarize, this paper introduces the following key contributions:
\begin{itemize}
    \item  {We formulate a multi-leader multi-follower Stackelberg game model to capture the interaction between RSUs and UAVs.
    To better measure the perception of metaverse services from UMUs, we innovatively incorporate a new immersion metric of UMUs into the utility function of UAVs. The immersion metric introduces a psychological factor into the formulation of utility functions, creating a more comprehensive and effective way to quantify the utility of UAVs. By integrating the immersion metric into the utilities of UAVs, our goal is to gain a more detailed understanding of the experience of UMUs in UAV metaverses. It can ultimately improve the adaptability of the utility function of UAV.}
   
    \item  {We propose a Tiny Multi-Agent Deep Reinforcement Learning (Tiny MADRL) algorithm based on pruning techniques, which is designed to approximate the Stackelberg equilibrium in a more computationally efficient and scalable manner. The tiny MADRL algorithm can efficiently derive the optimal game strategy in practical scenarios for achieving reliable UT migration.}
    By combining DRL with pruning techniques, our proposed algorithm can effectively handle intricate and dynamic environments, thus enhancing the overall performance in resolving the Stackelberg game for UT migration.
 
    \item  {We employ the Tiny MADRL algorithm to achieve optimal strategies of RSUs that converge faster and closer to the Stackelberg equilibrium than traditional DRL algorithms such as the Proximal Policy Optimization (PPO) algorithm.}
    This highlights the effectiveness and reliability of the proposed scheme.
\end{itemize}

The structure of the paper is as follows. Section \ref{Related} provides a review of the related work. In Section \ref{System}, we present the proposed framework for UT migration in UAV metaverses, which utilizes a tiny learning-based game approach and incorporates pruning techniques. This section also introduces the multi-leader multi-follower Stackelberg model. Section \ref{Algorithm} outlines the proposed Tiny MADRL algorithm, which aims to achieve the Stackelberg equilibrium and obtain an optimal game solution. The numerical results are presented in Section \ref{Results}. Finally, Section \ref{Conclusion} concludes this paper.

\section{Related Work}\label{Related}
\subsection{Metaverses}
The word \textit{``metaverses"} was first derived from the novel \textit{Snow Crash}, published in 1992, which depicted a virtual universe scene projecting the duality of the natural world and a copy of the digital environment \cite{stephenson2003snow}. The metaverse represents a convergence of physical and virtual spaces, surpassing the confines of reality to enable interactive engagement. This fusion is facilitated by a synthesis of cutting-edge technologies, encompassing AI, 6G wireless communication, cloud computing, XR, blockchain, IoT, and more \cite{10070406}. Collectively, the technologies empower individuals to seamlessly interact with computer-generated virtual elements within digital spaces while being physically present in the real-world \cite{10070406}.
In metaverses, users engage with virtual environments through personalized avatars, offering them an immersive experience that parallels their real-world existence \cite{lee2021all}.

The UAV metaverse is a novel concept that combines the UAV industry with the metaverse, creating a digital realm that is both virtual and interconnected with real-world UAVs. This innovative integration opens up a wide range of possibilities for emerging metaverse services tailored specifically for UMUs, such as disaster rescue, panoramic photography, environmental protection testing, and traffic monitoring. In the context of the UAV metaverses, advanced technologies are leveraged to enhance various applications and services. For instance, the UAV metaverses find application in panoramic photography, offering immersive and high-resolution aerial views. By integrating UAVs with the metaverses, UMUs can virtually explore and experience the stunning panoramic images, creating interactive and engaging experiences. Therefore, the UAV metaverses have the potential to shape the future of UAV-based services, enabling innovative and transformative solutions for smart farming, smart cities, and beyond.

\subsection{UAV Twins}
 {Combining UAVs with technologies such as Digital Twin (DT), UAV metaverses emerge as a manifestation of the digital economy era. These metaverses build a new way for traditional industries to explore new development avenues within the field of UAVs.
DTs are a virtual replica of the physical object, meticulously mapped in a virtual environment to accurately represent the entire lifespan of the corresponding entity\cite{zhongblockchain}. The authors in \cite{digital} proposed the use of DTs to bridge physical and virtual systems, where users representing the physical world can experience virtual world activities in real time. }
Likewise, UTs serve as exhaustive digital representations that encompass the entire lifecycle of UAVs, faithfully capturing the various stages and aspects of their existence \cite{Jinbo}. The UTs are tailored to address the specific needs and demands of UMUs. UTs play a pivotal role within the UAV metaverse ecosystem. Specifically, they encompass the physical attributes, functionalities, and operational characteristics of UAVs, effectively serving as comprehensive digital counterparts within the metaverses. Besides, they are created and maintained in digital environments, allowing UMUs to interact with and access UAV-related information and experiences. By utilizing UTs, UMUs can enjoy immersive virtual experiences within the UAV metaverses. For example, by using head-mounted displays or other Augmented Reality (AR) devices, UMUs can virtually navigate and control UAVs in real-world environments \cite{9880566}, enhancing situational awareness and enabling precise control and interaction with virtual UAVs.

\subsection{Resource Pricing Optimization in Metaverses}
Several endeavors have been undertaken to optimize resource pricing within metaverses. The authors in \cite{9880566} introduced a hierarchical game-theoretic framework for investigating the computational resource trading problem in vehicular metaverses. This framework established a coalition game at the upper stratum to discern dependable workers and introduced an incentive mechanism predicated on the Stackelberg game at the lower echelon, to engage the selected workers in rendering tasks.
In \cite{9838736}, the authors proposed a learning-based framework designed for incentivizing Virtual Reality (VR) services within metaverses. This framework employed an auction mechanism to ascertain the optimal pricing and allocation rules within the market context. 
The authors in \cite{huang2022joint} studied the service optimization for metaverses through distributed and centralized approaches. The paper presented a multi-leader multi-follower Stackelberg model that addresses the challenge of user association and resource pricing negotiation between users and Metaverse Service Providers (MSPs).
 {In \cite{huang}, the authors formulated a single-leader multi-follower Stackelberg game framework to investigate the economic aspects of users acquiring offloading services from an MSP and the optimal pricing strategy employed by the MSP.}
 
 {Although considerable attention has been given to optimizing resource pricing in metaverses, little research has focused on the resource optimization problem considering twin migration. To fill this gap, our paper aims to investigate and address the resource optimization challenge during UT migration in the UAV metaverses. }

\subsection{DRL with Pruning Techniques}
In \cite{8967118}, the authors employed game theory to devise a decentralized computation offloading algorithm, casting the problem as a partially observable Markov Decision Process (MDP) and resolving it through a DRL approach. This framework empowers users to discern optimal strategies directly from historical gameplay, obviating the necessity for prior knowledge of the behaviors of others.
DRL synergizes the benefits of deep learning and reinforcement learning, enabling the creation of algorithms that interact dynamically with the environment \cite{zhang2020q}. The algorithm then employs privacy-preserving techniques to learn iteratively, culminating in optimal decision-making. Given the inherent challenge of RSUs not having direct access to the privacy information of UAVs, utilizing the DRL algorithm becomes imperative to solving the Stackelberg equilibrium solution. This approach is crucial in scenarios where both RSUs and UAVs need to safeguard their privacy and refrain from disclosing all their information during interactions.

 {Nevertheless, the training of DRL models necessitates substantial computational resources and storage capacity. In response to the demand for lightweight DRL in certain scenarios, pruning techniques have been introduced to complement and enhance the capabilities of DRL. The authors in \cite{8962235} first addressed the computational and memory challenges of DRL by introducing Policy Pruning and Shrinking (PoPS), a novel iterative algorithm that effectively reduces redundancy in Deep Neural Networks (DNNs) for DRL applications. The authors in \cite{9727767} presented a novel framework for compressing DRL models, leveraging sparse regularized pruning techniques and policy shrinking technology to achieve a balance between high sparsity and compression rates, resulting in advancements in model size reduction and parameter optimization for DRL. 
The authors in \cite{li2022compact} proposed a compact DRL algorithm for dynamic off-chain routing in a Payment Channel Network (PCN)-based IoT context. They constructed an efficient DRL model by utilizing an adaptive pruning technique.}

 {Pruning techniques can be categorized into two main types: unstructured pruning \cite{li2022compact} and structured pruning \cite{rothnon}. Unstructured pruning involves directly pruning the weights of DRL models to achieve sparse weight matrices. On the other hand, structured pruning methods aim to remove entire sets of continuous parameters, such as weights or neurons, in a structured manner. In conclusion, pruning techniques, both structured and unstructured pruning, stand out as promising methods for compressing and accelerating DRL models, and more and more studies on the combination of DRL with pruning techniques are being conducted. Therefore, to expedite the attainment of results approximating the Stackelberg equilibrium, we have chosen to integrate the pruning techniques with MADRL, resulting in a novel algorithm called Tiny MADRL. This approach aims to enhance the efficiency and effectiveness of the solution process. }

\begin{table}[t]
  \begin{center}
    \caption{Key Notations in the Paper.}
    \label{notation}
    \begin{tabular}{l|l} 
    \toprule 
      \textbf{Notation} & \textbf{Definition}\\
      \hline
      $I$ & Number of UAVs  \\
      $J$ & Number of RSUs \\
      $p_i^j$ & The bandwidth selling price of RSU $j$ \\
      & to UAV $i$ \\ 
      $\boldsymbol{p}^j$ & The bandwidth selling price of RSU $j$ \\
      $\boldsymbol{p}^{-j}$ & The bandwidth selling price of RSUs \\
      &except RSU $j$ \\
      $b_i^j$ & The purchase bandwidth of UAV $i$ from RSU $j$ \\
      $\boldsymbol{b}_i$ & The purchase bandwidth of UAV $i$ from RSUs \\
      $V_j(\boldsymbol{p}^j,\boldsymbol{p}^{-j},\boldsymbol{b}^j)$ & The utility of RSU $j$ \\
      $U_i(\boldsymbol{b}_i,\boldsymbol{p}_i)$ & The utility of UAV $i$ \\
      $\sigma^{j2}$ &  Additive white Gaussian noise of RSU $j$   \\
      & transmitting a rendered screen to UAVs \\
      $g^j$ & Channel gain from RSU $j$ transmitting a rendered \\
      & screen to UAVs \\
      $m^j$ & Transmit power from RSU $j$ transmitting a \\
      & rendered screen to UAVs \\
      $c^j$ & The bandwidth cost of RSU $j$ \\
      $\bar{p}^j$ & The upper bandwidth selling price of \\
      & RSU $j$  \\
      $R_i$ & The payment budget of UAV $i$ \\
      $SSIM_i^{th}$ & Threshold of minimum SSIM of UAV $i$ \\
      $l_i^j$ & The brightness similarity between the rendered  \\
      & image in RSU $j$ and the received image by UMUs \\
      $c_i^j$ & The contrast similarity between the rendered  \\
      & image in RSU $j$ and the received image by UMUs \\
      $s_i^j$ & The structure similarity between the rendered  \\
      & image in RSU $j$ and the received image by UMUs \\
      \bottomrule
    \end{tabular}
  \end{center}
\end{table}

\section{System Model and Problem Formulation}\label{System}

\subsection{System Model}

\begin{figure*}[ht]
\centering
\includegraphics[width=0.95\textwidth]{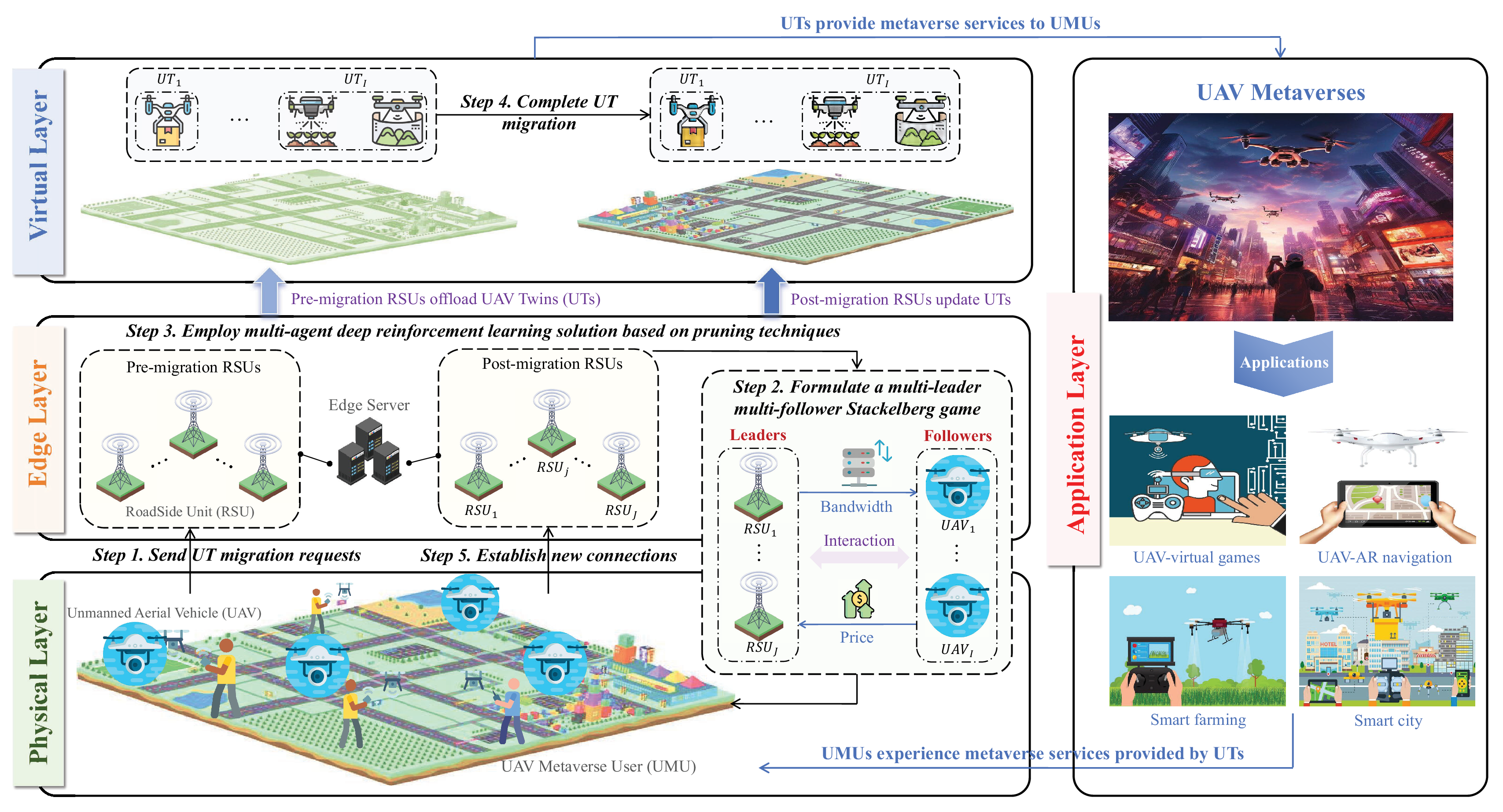}
\caption{A tiny learning-based game approach framework based on the pruning techniques. Note that UAVs seamlessly migrate their UTs from pre-migration RSUs to post-migration RSUs, ensuring that UMUs can consistently access and benefit from the metaverses services offered by the UTs.}
\label{Stac}
\end{figure*}

The Stackelberg game has been widely used for resource management solutions \cite{zhu2021deep, zhongblockchain,zhang2023learning}. 
Due to the dynamic mobility exhibited by UAVs and the inherent limitations in communication coverage offered by RSUs, the decision-making process of UAVs is to select the appropriate RSUs to migrate their UTs. The decision is affected by their trajectories and the pricing schemes, offered by the RSUs they choose to connect. Therefore, the interaction between post-migration RSUs and UAVs can be viewed as resource management. As shown in Fig. \ref{Stac}, the UT migration process in UAV metaverses is described as follows:

\emph{\textbf{Step 1. Send UT migration requests to the pre-migration RSUs:}} Before UTs are migrated, the UAVs need to send migration requests to the RSUs that currently deploy their corresponding UTs, where the RSUs are called the pre-migration RSUs. 

\emph{\textbf{Step 2. Formulate a multi-leader multi-follower Stackelberg game:}} After sending the UT migration requests, the UAVs need to choose to which RSU to migrate their UTs. In the selection process, the RSUs publish their resource pricing strategies, and the UAVs decide how many resource units to purchase for efficient UT migration based on the bandwidth prices of the RSUs. Therefore, a multi-leader multi-follower Stackelberg game between RSUs and UAVs for UT migration can be formulated. 

 {\emph{\textbf{Step 3. Employ multi-agent deep reinforcement learning solution based on the pruning technique:}} Then, the Tiny MADRL algorithm is deployed on edge servers to determine an optimal solution to the Stackelberg game, aiming to facilitate efficient and dependable UT migration. }

 {\emph{\textbf{Step 4. Complete UT migration from the pre-migration RSUs to the post-migration RSUs:}} Through the proposed Stackelberg model, the UAVs can finally choose to migrate their UTs from the pre-migration RSUs to the post-migration RSUs. In this Stackelberg model, the post-migration RSUs act as leaders in the Stackelberg game, taking the responsibility of allocating bandwidth resources for UT migration and autonomously determining pricing strategies for the available bandwidth. Note that the RSUs to which UAV $i$ intends to migrate its UTs are referred to as post-migration RSUs.}

\emph{\textbf{Step 5. Establish new connections with the post-migration RSUs:}} After completing the UT migration, the UAVs establish new connections with the post-migration RSUs, which enables their UTs to obtain resource services provided by the post-migration RSUs, thereby providing uninterrupted metaverse service experiences for UMUs. 

Our main work is to study \emph{\textbf{Step 2}} and \emph{\textbf{Step 3}}, i.e., formulating a multi-leader multi-follower Stackelberg game based on the interactions between RSUs and UAVs, and utilizing the Tiny MADRL algorithm to determine the optimal solution that ensures the attainment of efficient and reliable UT migration.

\subsection{Multi-leader Multi-Follower Stackelberg Game between RSUs and UAVs }\label{Stackelberg}

 {This paper examines a UT migration system comprising $J$ RSUs and $I$ UAVs. The set of RSUs is denoted as $\mathcal{J}\triangleq\{1,2,\ldots,j,\ldots, J\}$, while the set of UAVs is represented as $\mathcal{I}\triangleq\{1,2,\ldots,i,\ldots,I\}$. Each UAV has corresponding UTs to manage their metaverse applications \cite{Jinbo}. We propose a multi-leader multi-follower Stackelberg game framework wherein RSUs assume the role of leaders and UAVs function as followers. 
The key notations used in the problem formulation are shown in Table \ref{notation}. Given that UMUs may access various virtual services, e.g., UAV-AR navigation and UAV virtual games \cite{sefercik2022creation}, a single UAV can have numerous UTs. As a result, UAVs can migrate their UTs to different RSUs and acquire resources from multiple sources. We define $b_i^j$ as the bandwidth demand that UAV $i$ requests from RSU $j$. $p_i^j$ is denoted as the bandwidth price, representing the monetary payments of UAV $i$ using per unit of bandwidth resource from RSU $j$. Both the RSUs and UAVs optimize their utilities by adjusting their strategies. We use backward induction to derive utility functions for RSUs and UAVs. Therefore, the game is decomposed into two sub-games as follows.}

\subsubsection{UAVs' Bandwidth Demands in Stage II}
In Stage II, UAV $i$ determines the bandwidth demand, denoted as $b_i^j$, to purchase from RSU $j$. This purchase decision is made to facilitate the migration of the corresponding UT. The bandwidth demand is guided by the bandwidth price strategies issued by the RSUs in Stage I. We denote the transmit power from RSU $j$ transmitting a rendered screen to UAVs as $m^j$. Let $\tau^j$ be the Signal-to-Noise Ratio (SNR) from RSU $j$ transmitting a rendered screen to UAVs, which is expressed as \cite{ren2022quantum}
\begin{equation}
    \tau^j=\frac{m^jg^j}{(\sigma^{j})^2},
\end{equation}
where $g^j$ and $\sigma^{j}$ represent the channel gain and additive white Gaussian noise, respectively. We define the spectrum efficiency of RSU $j$ transmitting a rendered screen to UAVs as $q^j$, which is expressed as
\begin{equation}
    q^j=\log_2(1+\tau^j).
\end{equation}
Therefore, the transmission rate of RSU $j$ transmitting a rendered screen to UAV $i$ is given by \cite{ren2022quantum}
\begin{equation}
    r_i^j=b_i^jq^j.
\end{equation}

The authors in \cite{1284395} introduced the Structural Similarity Measure (SSIM) to assess local patterns of pixel intensities, accounting for normalization in terms of brightness and contrast. SSIM is a well-known and widely-adopted quality metric used to quantify the degree of similarity between the structural information of two images, which is a combination of three components and expressed as \cite{1284395}
\begin{equation}
    SSIM(x_1,x_2)=g\big(l(x_1,x_2),c(x_1,x_2),s(x_1,x_2)\big),
\end{equation}
where $l(x_1,x_2)$, $c(x_1,x_2)$, and $s(x_1,x_2)$ represent luminance comparison, contrast comparison, and structure comparison between image $x_1$ and $x_2$, respectively. $g(\cdot)$ represents the combination function of SSIM. Note that the three components of SSIM are strictly independent, and SSIM meets the following three conditions \cite{1284395}: 
\begin{itemize}
    \item Symmetry: The SSIM between two images, denoted as $x_1$ and $x_2$, is symmetric, meaning that $SSIM(x_1,x_2)=SSIM(x_2,x_1)$.
    \item Boundedness: The SSIM value between two images is bounded by $1$, indicating that the SSIM score is always less than or equal to $1$, i.e., $SSIM(x_1,x_2)\leq{1}$.
    \item Unique maximum: The SSIM score reaches its maximum value of $1$ if and only if the images being compared, $x_1$ and $x_2$, are identical, i.e., $SSIM(x_1,x_2)=1$ implies $x_1=x_2$.
\end{itemize}

 {In the quest for an immersive user experience within the metaverses, real-time rendering emerges as a pivotal technology \cite{9880566}, with graphics rendering standing as a major function. With the dynamic mobility of UAVs, UAVs migrate their UTs between RSUs to make UMUs access metaverses services continuously and seamlessly. The fidelity of image rendering assumes a pivotal role in shaping the immersive experiences of UMUs.
Consequently, we adopt the SSIM to gauge the realism of picture rendering, ensuring a seamless virtual environment for UMUs. We denote $SSIM_i^j$ as the SSIM score of the image received by UAV $i$ deciding to migrate its UTs to RSU $j$, which is given by 
\begin{equation}
    SSIM_i^j = (l_i^j)^\alpha (c_i^j)^\beta (s_i^j)^\nu,
\end{equation}
where $l_i^j$, $c_i^j$, and $s_i^j$ represent the brightness similarity, contrast similarity, and structural similarity, respectively. The values quantify the image comparison between the rendered image in RSU $j$ and the received image by UMU after the migration of UT of UAV $i$ to RSU $j$, followed by the transmission to UMU. $\alpha>0$, $\beta>0$, and $\nu>0$ are the parameters that adjust the brightness, contrast, and structural weights \cite{1284395}, respectively. Without loss of generality, to simplify the presentation of derivation, we set $\alpha=\beta=\nu=1$. }

We denote $Sat_i^j$ as the satisfaction of UAV $i$ selecting RSU $j$ for UT migration, and employ logarithmic function $\ln(\cdot)$ to model the effect of the image transmission rate on the satisfaction of UAV $i$, which is expressed as \cite{ou2023stackelberg}
\begin{equation}
    Sat_i^j=\mu_i\ln(1+r_i^j),
\end{equation}
where $\mu_i$ denotes the satisfaction factor of UAV $i$, representing the sensitivity of UAV $i$ to the image transmission rate. Drawing inspiration from Gustav Fechner's extension of the relationship between human perception and relative stimulus changes, known as Weber-Fechner's Law (WFL) \cite{reichl2013logarithmic}, we leverage the $\ln(\cdot)$ function to model human perception of service quality. Then, we propose the utilization of metaverse immersion metrics to model the Quality of Experience (QoE) of UMUs based on WFL \cite{du}. Therefore, the level of metaverses immersion that UAV $i$ can provide to UMUs after migrating its UTs to RSU $j$ can be expressed as 
\begin{equation}
    I_i^j=Sat_i^j\ln\bigg(\frac{SSIM_i^j}{SSIM_{i}^{th}}\bigg), 
\end{equation}
where $SSIM_{i}^{th}$ represents the minimum SSIM threshold of UAV $i$ \cite{ yu2023attention}.

For UAVs, the metaverse immersion metric serves as a quantitative measure of QoE for UMUs, and a higher metric value indicates an enhanced QoE. Therefore, if UAV $i$ requests bandwidth resources from RSU $j$, the utility function of UAV $i$ for RSU $j$ is given by
\begin{equation}
\begin{split}
    U_i^j&=\xi_iI_i^j-p_i^j{b_i^j}, \\
    &=\delta_i\ln(1+b_i^jq^j)\ln\bigg(\frac{SSIM_i^j}{SSIM_{i}^{th}}\bigg)-p_i^j{b_i^j},
\end{split}
\end{equation}
where $\xi_i$ represents the metaverse immersion factor of UAV $i$, and $\delta_i=\xi_i\mu_i$ for ensuring $U_i^j>0$. Since we consider multiple RSUs that provide similar resource services, UAV $i$ can choose to request bandwidth resources from any RSU, and the utility function of UAV $i$ is expressed as
\begin{equation}
    U_i(\boldsymbol{b}_i,\boldsymbol{p}_i)=\sum_{j\in{\mathcal{J}}}U_i^j,
\end{equation}
where $\boldsymbol{b}_i=\{b_i^1,\ldots,b_i^j,\ldots,b_i^J\}$ represents the bandwidth demand of UAV $i$ from each RSU $j$, and $\boldsymbol{p}_i=\{p_i^1,\ldots,p_i^j,\ldots,p_i^J\}$ represents the bandwidth price of each RSU $j$. 
Each UAV $i$ strategically determines its optimal bandwidth demand strategy vector, represented as $\boldsymbol{b_i}$. This strategic decision-making process is driven by the objective of maximizing their respective utilities $U_i$. The optimization is performed by the bandwidth price policy vector $\boldsymbol{p}_i$. The optimization problem for bandwidth purchases of UAVs is formulated as follows:
\begin{equation}
    \begin{split}
    \textbf{Problem:}\:&\max\limits_{\boldsymbol{b}_i}\:U_i(\boldsymbol{b}_i,\boldsymbol{p}_i)  \\
    &\:\:\text{s.t.}\:\: 0 \leq b_i^j,\:\forall j, \\
    &\quad\:\:\:\:\textstyle\sum_{j\in\mathcal{J}}p_i^j{b_i^j}\leq{R_i},
    \end{split}
    \label{U_i}
\end{equation}
where $R_i$ is the payment budget of UAV $i$. UAVs decide their bandwidth purchase strategies depending on the pricing strategies employed by RSUs. The cumulative bandwidth purchases made by UAVs, on the other hand, influence the selling price strategies of RSUs. This interaction lays the groundwork for our upcoming derivation of the utility function of RSUs.

\subsubsection{RSUs' bandwidth selling prices in Stage I}
In Stage I, each RSU $j$ formulates its pricing strategy for the bandwidth to UAV $i$, denoted as $p_i^j$ \cite{xiong2019cloud}. This pricing strategy is determined based on an analysis of the anticipated bandwidth purchase demand from UAVs. 

 {To facilitate UT migration, RSUs are required to purchase bandwidth from the system market. Subsequently, they allocate sufficient bandwidth resources to UAVs for UT migration, which results in associated costs. We denote the bandwidth cost as $c^j$, which is determined by the market regulation for RSU $j$ \cite{lyu}. Therefore, the utility of RSU $j$ encompasses two distinct components, namely the revenue and the cost, which can be expressed as}
\begin{equation}
\label{V}
    V_j(\boldsymbol{p}^j,\boldsymbol{p}^{-j},\boldsymbol{b}^j)=\sum_{i\in{\mathcal{I}}}\big({p_i^j{b_i^j}}-c^jb_i^j\big),
\end{equation}
where $\boldsymbol{p}^j=\{p_1^j,\ldots,p_i^j,\ldots, p_I^j\}$ and $\boldsymbol{p}^{-j}$ denote the vector of bandwidth price strategies posted by RSU $j$ and all RSUs except RSU $j$, respectively. $\boldsymbol{b}^j=\{b_1^j,\ldots,b_i^j,\ldots, b_I^j\}$ represents the bandwidth demand that each UAV $i$ purchases from RSU $j$. Note that each RSU $j$ decides to maximize its utility function $V_j$ according to the policies of all UAVs (i.e., $\boldsymbol{b^j}$), thus obtaining the optimal strategy $p_i^j$. In addition, we also need to consider the constraint for the utility function of RSUs, which is the price constraint, expressed as ${c^j}\leq{p_i^j}\leq{\bar{p}_{i}^j}$, where $\bar{p}^j$ represents the upper bandwidth price, which is determined by the market regulations for RSU $j$ \cite{ou2023stackelberg}. Therefore, the bandwidth price optimization problem of the RSUs as the leaders is expressed as follows:
\begin{equation}
    \begin{split}
    \textbf{Problem:}\:&\max\limits_{\boldsymbol{p}^j}\:V_j(\boldsymbol{p}^j,\boldsymbol{p}^{-j},\boldsymbol{b}^j)  \\
    &\:\:\text{s.t.}\:\: c^j \leq p_i^j \leq \bar{p}^j.
    \end{split}
    \label{V_j}
\end{equation}

In consideration of the utility functions of RSUs and UAVs, the application of the backward induction method is employed to analyze the multi-leader multi-follower Stackelberg equilibrium in the subsequent analysis.

\subsubsection{Stackelberg Equilibrium Analysis}
In the Stackelberg equilibrium, the leaders (i.e., the RSUs) cannot increase profits by changing bandwidth prices, and the followers (i.e., the UAVs) are incapable of augmenting their utilities by adjustments in bandwidth demands. 
In the context of the framework involving RSUs and UAVs, engaged in a multi-leader multi-follower Stackelberg game, the definition of the Stackelberg equilibrium unfolds as follows.

\begin{definition}
     {\textbf{(Stackelberg equilibrium):} Let $(b_i^j)*$ and $(p_i^j)^*$ denote the optimal bandwidth demand of UAV $i$ to RSU $j$, and the optimal bandwidth price of RSU $j$ to UAV $i$, respectively. Therefore, $\big((b_i^j)*, (p_i^j)^*,\:\forall i\in{\mathcal{I}},\:\forall j\in{\mathcal{J}}\big)$ is the Stackelberg equilibrium within the context of the Stackelberg game involving RSUs and UAVs is contingent upon strict satisfying the following two conditions \cite{8758205}:
\begin{enumerate}[1)]
    \item There exists a subgame among the RSUs, which strictly compete for bandwidth prices in a non-cooperative fashion. The Nash equilibrium of the subgame is achieved, i.e., $({p}_i^j)^*$ is a Nash equilibrium for RSUs, if the following condition is satisfying,
    \begin{equation}
    V_j\big(\boldsymbol{p}^{j*},\boldsymbol{p}^{-j*},\boldsymbol{b}^{j*}\big)\ge{V_j\big(\boldsymbol{p}^j,\boldsymbol{p}^{-j*},\boldsymbol{b}^{j*}\big)},\:\forall{j},
\end{equation}
\item Under the pricing strategy denoted by $\boldsymbol p_i^*$, the optimal reaction elicited from the utility function of UAV $i$ is represented by $\boldsymbol b_i^*$. This optimal response,  $\boldsymbol b_i^*$, serves as the exclusive solution that maximizes the utility function $U_i(\boldsymbol{b}_i,\boldsymbol{p}_i^*)$, conditioned upon the pricing strategy $\boldsymbol p_i^*$.
\end{enumerate}
}
\end{definition}
For the ease of representation, we represent $\ln\Big(\frac{SSIM_i^j}{SSIM_i^{th}}\Big)$ as $S_i^j$ in the following.  

\paragraph{UAVs' optimal strategies in Stage II}\label{UMU}
The UAVs determine the profit-maximizing optimal bandwidth purchase strategies in Stage II by considering the published bandwidth price policy vector from the RSUs.

\begin{Proposition}
For any given RSU $j$ with the bandwidth selling price $p_i^j$, the optimization problem of UAV $i$ is inherently convex. The optimal strategy dictating the bandwidth demand to be procured by UAV $i$ can be formulated as
\begin{equation}\label{b*}
    b_i^{j*}=\begin{cases}
    \check{b}_i^j,\quad \mathrm{if} \:\big(\delta_iq^jS_i^j>p_i^j\big)\bigcap\big(\sum_{j\in{\mathcal{J}}}p_i^j\check{b}_i^j\leq{R_i}\big), \\
    \hat{b}_i^j,\quad \mathrm{if} \bigg(p_i^j<\frac{S_i^jq^j\big(R_i+\sum_{\forall k\in\mathcal{J}\backslash j}\frac{p^k}{q^k}\big)}{\sum_{\forall k\in\mathcal{J}\backslash j}S_i^k}\bigg) \\
    \quad\quad\quad\quad\quad\quad\:\bigcap\big(\sum_{j\in{\mathcal{J}}}p_i^j\check{b}_i^j>R_i\big), \\
    0, \quad\: \mathrm{otherwise},
    \end{cases}
\end{equation}
where $\check{b}_i^j$ and $\hat{b}_i^j$ represent the optimal strategy of UAV $i$ to RSU $j$ if the second constraint in Eq. (\ref{V_j}) is inactive and active, respectively, which are expressed as
\begin{equation}
\begin{split}
    \check{b}_i^j&=\frac{\delta_iS_i^j}{p_i^j}-\frac{1}{q^j},\\
    \hat{b}_i^j&=\frac{S_i^j\big(R_i+\sum_{j\in{\mathcal{J}}}\frac{p_i^j}{q^j}\big)}{p_i^j\sum_{j\in{\mathcal{J}}}S_i^j}-\frac{1}{q^j}.
\end{split}
\end{equation}
\end{Proposition}

\begin{proof}
The derivatives of the utility function $U_i$ with respect to the bandwidth demand strategy $b_i^j$ up to the first and second orders are provided as
\begin{equation}
\begin{split}
    \frac{\partial{U_i}}{\partial{b}_i^j}&=\sum_{j\in{\mathcal{J}}}\bigg(\frac{\delta_iq^jS_i^j}{1+b_i^jq^j}-p_i^j\bigg),\\
    \frac{\partial^2U_i}{\partial{b}_i^{j2}}&=-\sum_{j\in{\mathcal{J}}}\bigg(\frac{\delta_iq^{j2}S_i^j}{(1+b_i^jq^j)^2}\bigg)<{0}.
\end{split}
\end{equation}  
The utility function $U_i$ demonstrates a singular point at which its first-order derivative equals zero. Furthermore, the negative second-order derivative indicates the strict concavity of the utility function $U_i$ to the bandwidth demand strategy $b_i^j$ employed by UAV $i$. Thus, the optimization problem for UAV $i$ is a convex optimization problem \cite{9830075}.

To determine the optimal response policy $\check{b}_i^j$ of UAV $i$ with respect to RSU $j$, we equate the first-order derivative of $U_i$ to zero. We can obtain the best response policy $\check{b}_i^j$ shown as
\begin{equation}\label{b_bar}
    \check{b}_i^j=\frac{\delta_iS_i^j}{p_i^j}-\frac{1}{q^j}.
\end{equation}
However, if $\delta_iq^jS_i^j\leq{p_i^j}$, UAV $i$ will not purchase bandwidth resources from RSU $j$, and then $b_i^j=0$. Therefore, if the second constraint in Eq. (\ref{U_i}) is inactive, the best response policy of UAV $i$ for RSU $j$ is
\begin{equation}\label{inactive}
    b_i^{j*}=
    \begin{cases}
        \check{b}_i^j,&\mathrm{if} \: p_i^j<\delta_iq^jS_i^j, \\
        0,&\mathrm{if} \:p_i^j\ge\delta_iq^jS_i^j.
    \end{cases}
\end{equation}

If the second constraint of Eq. (\ref{U_i}) is active, we have $\sum_{i\in{\mathcal{I}}}p_i^jb_i^j\ge R_i$. Then, the Lagrangian function can be utilized to find the optimal strategies for UAV $i$, given by
\begin{equation}
    L_i=U_i-\lambda_i\bigg(\sum_{j\in{\mathcal{J}}}p_i^jb_i^j-R_i\bigg),
\end{equation}
where $\lambda_i$ represents the Lagrangian multiplier \cite{9830075}. The Karush-Kuhn-Tucker (KKT) conditions for problem Eq. (\ref{U_i}) is given by
\begin{equation}
    \frac{\partial{L}_i}{\partial{b}_i^j}=\frac{\delta_iq^jS_i^j}{1+b_i^jq^j}-p_i^j-\lambda_ip_i^j=0, \label{KKT}
\end{equation}
\begin{equation}
    \lambda_i\bigg(\sum_{j\in{\mathcal{J}}}p_i^jb_i^j-R_i\bigg)=0,
\end{equation}
\begin{equation}
    \lambda_i\ge{0},\quad \sum_{j\in{\mathcal{J}}}p_i^jb_i^j=R_i. \label{lambda}
\end{equation}
By solving Eq. (\ref{KKT}), we can obtain the optimal strategies of UAV $i$ for RSU $j$, which is expressed as
\begin{equation}
    b_i^j=\frac{\delta_iS_i^j}{p_i^j(1+\lambda_i)}-\frac{1}{q^j}.\label{bij}
\end{equation}
The substitution of Eq. (\ref{bij}) into Eq. (\ref{lambda}) yields the following expression
\begin{equation}
    \sum_{j\in{\mathcal{J}}}\bigg(p_i^j\bigg(\frac{\delta_iS_i^j}{p_i^j(1+\lambda_i)}-\frac{1}{q^j}\bigg)\bigg)=R_i,\label{*}
\end{equation}
\begin{equation}
    \lambda_i=\frac{\delta_i\sum_{j\in{\mathcal{J}}}S_i^j}{R_i+\sum_{j\in{\mathcal{J}}}\frac{p_i^j}{q^j}}-1.\label{lambda_i}
\end{equation}
By substituting Equation (\ref{lambda_i}) into Equation (\ref{bij}), the optimal strategy for UAV $i$ can be obtained by
\begin{equation}
\begin{split}
    \hat{b}_i^j&=\frac{S_i^j\big(R_i+\sum_{j\in{\mathcal{J}}}\frac{p_i^j}{q^j}\big)}{p_i^j\sum_{j\in{\mathcal{J}}}S_i^j}-\frac{1}{q^j}, \\
    &=\frac{S_i^j\big(R_i+\sum_{\forall k\in\mathcal{J}\backslash j}\frac{p^k}{q^k}+\frac{p_i^j}{q^j}\big)}{p_i^j\sum_{j\in{\mathcal{J}}}S_i^j}-\frac{1}{q^j}.
\end{split}
\end{equation}
We denote $\sum_{j\in{\mathcal{J}}}{S_i^j}$ and $\sum_{\forall k\in\mathcal{J}\backslash j}\frac{p^k}{q^k}$ as $A$ and $B$, respectively. Therefore, we can achieve the optimal strategy of UAV $i$ to RSU $j$ as
\begin{equation}\label{b_hat}
\begin{split}
    \hat{b}_i^j&=\frac{S_i^j\big(R_i+B+\frac{p_i^j}{q^j}\big)}{Ap_i^j}-\frac{1}{q^j},\\ 
    &=\frac{S_i^j(R_i+B)}{Ap_i^j}+\frac{S_i^j}{Aq^j}-\frac{1}{q^j}.
\end{split}
\end{equation}
If $\frac{S_i^jq^j\big(R_i+\sum_{\forall k\in\mathcal{J}\backslash j}\frac{p^k}{q^k}\big)}{\sum_{j\in{\mathcal{J}}}S_i^j}\leq{p_i^j}$, UAV $i$ will not purchase bandwidth from RSU $j$. Therefore, the optimal strategy of UAV $i$ to RSU $j$ is expressed as
\begin{equation}\label{active}
    b_i^{j*}=
    \begin{cases}
        \hat{b}_i^j,\quad \mathrm{if} \: p_i^j<\frac{S_i^jq^j\big(R_i+\sum_{\forall k\in\mathcal{J}\backslash j}\frac{p_i^k}{q^k}\big)}{\sum_{\forall k\in\mathcal{J}\backslash j}S_i^k}, \\
        0,\quad\:\: \mathrm{if} \:p_i^j\ge\frac{S_i^jq^j\big(R_i+\sum_{\forall k\in\mathcal{J}\backslash j}\frac{p_i^k}{q^k}\big)}{\sum_{\forall k\in\mathcal{J}\backslash j}S_i^k}.
    \end{cases}
\end{equation}

\end{proof}
The above results indicate that the higher the bandwidth price set by RSU $j$, the less amount of bandwidth is purchased by UAV $i$.
Given the aforementioned analysis, the UAVs can adapt their optimal bandwidth purchase strategies in response to the bandwidth selling price offered by RSUs, as shown in \textbf{Algorithm \ref{1}} \cite{9830075}.

\begin{algorithm}[t]
\label{1}
\DontPrintSemicolon
  \SetAlgoLined
  \KwIn {The bandwidth selling price $\boldsymbol{p}^j$ released by all RSUs.}
  \KwOut {The optimal bandwidth demand strategies $b_i^{j*}$ of each UAV $i$ to each RSU $j$.}
  \eIf {$\sum_{j\in{\mathcal{J}}}p_i^j\big(\frac{\delta_iS_i^j}{p_i^j}-\frac{1}{q^j}\big)\leq{R_i}$}{
  \tcp{\rm{ If $\sum_{j\in{J}}p_i^j{b_i^j}\leq{R_i}$ is inactive.}} \
    Calculate $b_i^{j*}$ based on Eq. (\ref{inactive}).\;
  }{
  \tcp{\rm{If $\sum_{j\in{J}}p_i^j{b_i^j}\leq{R_i}$ is active.}} \
    Calculate $b_i^{j*}$ based on Eq. (\ref{active}).\;
  }
  \caption{ {Find the optimal bandwidth demands for UAVs.}}
\end{algorithm}

\paragraph{RSUs' optimal strategies as equilibrium in Stage I}
By predicting the optimal bandwidth purchase policies of UAVs, RSUs take the role of leaders in Stage I, wherein they aim to optimize their utilities. Non-cooperative games place a strong emphasis on individual rationality and the pursuit of optimal decisions by each player. In the context of RSU decision-making, each RSU participating in UT migration acts in a self-interested and rational manner, making strategic choices with the primary goal of maximizing its utility. Consequently, we formulate the competitive dynamics among RSUs as a non-cooperative game, and the sought-after solution to this game is the Nash equilibrium \cite{8758205}.

\begin{lemma}\label{Nash equilibrium}
The existence of Nash equilibrium in a non-cooperative game can be guaranteed when the following three conditions are met \cite{zhan2020learning}:
\begin{itemize}
    \item The player set is characterized by finiteness.
    \item Both strategy sets are delineated by closure and boundedness, demonstrating convexity.
    \item The utility functions exhibit continuity and quasi-concavity within the confines of the strategy space.
\end{itemize}
\end{lemma}

\begin{theorem}\label{non}
    There exists a Nash equilibrium in the non-cooperative game among RSUs \cite{8758205}.
\end{theorem}

\begin{proof}
When the second constraint of Eq. (\ref{U_i}) is active, the optimal strategies of UAV $i$ (as given in Eq. (\ref{b_hat})) are substituted into the utility function of RSU $j$ as
\begin{equation}
    V_j=\sum_{i\in{\mathcal{I}}}\bigg(\big(p_i^j-c^j\big)\bigg(\frac{S_i^j(R_i+B)}{Ap_i^j}+\frac{S_i^j}{Aq^j}-\frac{1}{q^j}\bigg)\bigg).
\end{equation}
By computing the first and second-order derivative of $V_j$ with respect to $p_i^j$, the following expressions are derived, i.e.,
\begin{equation}\label{first_active}
    \frac{\partial V_j}{\partial p_i^j}=\sum_{i\in\mathcal{I}}\bigg(\frac{S_i^j-A}{Aq^j}+\frac{c^jS_i^j(R_i+B)}{A({p_i^j})^2}\bigg),
\end{equation}
\begin{equation}
    \frac{\partial V^2_j}{\partial (p_i^j)^2}=-2\sum_{i\in\mathcal{I}}\bigg(\frac{c^jS_i^j(R_i+B)}{A(p_i^j)^3}\bigg)<0.
\end{equation}
Then, by setting the first-order derivative of $V_j$ as $0$, we can obtain the optimal strategy of RSU $j$ to UAV $i$ expressed as
\begin{equation}\label{phi_active}
\begin{split}
    \phi_i^j(\boldsymbol{p})&=\sqrt{\frac{q^jc^jS_i^j(R_i+B)}{A-S_i^j}}, \\
    &=\sqrt{\frac{q^jc^jS_i^j\big(R_i+\sum_{\forall k\in\mathcal{J}\backslash j}\frac{p_i^k}{q^k}\big)}{\sum_{j\in{\mathcal{J}}}S_i^j-S_i^j}}, \\
    &=\sqrt{\frac{q^jc^jS_i^j\big(R_i+\sum_{\forall k\in\mathcal{J}\backslash j}\frac{p_i^k}{q^k}\big)}{\sum_{\forall k\in\mathcal{J}\backslash j}S_i^k}},
\end{split}
\end{equation}
where $\boldsymbol{p}=\sum_{\forall k\in\mathcal{J}\backslash j}p_i^k$ represents the strategies of RSUs expect RSU $j$, which indicates that RSU $j$ determines its strategy considering the strategies of other RSUs. 

In the case where the second constraint of Eq. (\ref{U_i}) is inactive, we substitute the optimal strategies of UAV $i$ (as specified in Eq. (\ref{b_bar})) to the utility function of RSU $j$ as
\begin{equation}
    V_j=\sum_{i\in\mathcal{I}}\bigg(\big(p_i^j-c^j\big)\bigg(\frac{\delta_iS_i^j}{p_i^j}-\frac{1}{q^j}\bigg)\bigg).
\end{equation}
The expressions for the first-order and second-order derivatives of $V_j$ with respect to $p_i^j$ are given by
\begin{equation}
    \frac{\partial V_j}{\partial p_i^j}=\sum_{i\in\mathcal{I}}\bigg(-\frac{1}{q^j}+\frac{\delta_ic^jS_i^j}{(p_i^j)^2}\bigg),
\end{equation}
\begin{equation}
    \frac{\partial V^2_j}{\partial (p_i^j)^2}=-2\sum_{i\in\mathcal{I}}\frac{\delta_ic^jS_i^j}{(p_i^j)^3}<0.
\end{equation}
Similarly, equating the first-order derivative of $V_j$ to $0$, the optimal strategy of RSU $j$ to UAV $i$ is expressed as
\begin{equation}
    p_i^{j*}=\sqrt{\delta_iS_i^jq^jc^j}.
\end{equation}

 {Note that the second-order derivative of the utility function $V_j$ for RSU $j$ is negative, indicating its concavity. Furthermore, the strategy sets for the RSUs meet the essential criteria of being closed, bounded, and convex. Additionally, considering the finite nature of the RSU set and the continuity of their utility functions, the Nash equilibrium among the RSUs exists and is unique according to \textbf{Lemma \ref{Nash equilibrium}}. }

\end{proof}

\begin{theorem}
    A Stackelberg equilibrium is present in the multi-leader multi-follower Stackelberg game between RSUs and UAVs \cite{8758205}.
\end{theorem}

\begin{proof}
The existence of the Nash equilibrium among RSUs has been formally established, as evidenced by \textbf{Theorem \ref{non}}  \cite{8758205}. Moreover, in Section \ref{UMU}, it has been demonstrated that UAVs possess the capability to select optimal strategies, maximizing their utilities for the given pricing strategies of all RSUs. Consequently, by considering the outcomes, it can be deduced that the Stackelberg game between RSUs and UAVs attains a Stackelberg equilibrium. 
\end{proof}

\begin{lemma}
    A Stackelberg equilibrium is guaranteed to be unique if the following three properties are satisfied \cite{9492053}:
    \begin{itemize}
        \item Positivity: $\phi_i^j(\boldsymbol{p})>0$.
        \item Monotonicity: If $\boldsymbol{p}'>\boldsymbol{p}$, then $\phi_i^j(\boldsymbol{p}')>\phi_i^j(\boldsymbol{p})$.
        \item Scalability: For $\forall \chi>1$, then $\chi\phi_i^j(\boldsymbol{p})>\phi_i^j(\chi\boldsymbol{p})$.
    \end{itemize}
    where $\phi_i^j(\boldsymbol{p})$ represents the optimal strategy of RSU $j$.
\end{lemma}

\begin{theorem}
    The Stackelberg equilibrium among RSUs and UAVs is unique.
\end{theorem}

\begin{proof}
Based on Eq. (\ref{phi_active}), it is clear that $\phi_i^j(\boldsymbol{p})>0$ is satisfied. Moreover, $\phi_i^j(\boldsymbol{p})$ is an increasing function for $\sum_{\forall k\in\mathcal{J}\backslash j}p_i^k$, i.e., $\phi_i^j(\boldsymbol{p})$ increases as $\boldsymbol{p}$ increases. Therefore, when $\boldsymbol{p}'>\boldsymbol{p}$, i.e., $\sum_{\forall k\in\mathcal{J}\backslash j}{p_i^k}'>\sum_{\forall k\in\mathcal{J}\backslash j}p_i^k$, we can get $\phi_i^j(\boldsymbol{p}')>\phi_i^j(\boldsymbol{p})$. Finally, since $\chi>1$, then $\chi^2>\chi$. Therefore, for $\forall \chi>1$, we have
\begin{equation}
    \begin{split}
        \chi\phi_i^j(\boldsymbol{p})&=\chi\sqrt{\frac{q^jc^jS_i^j\big(R_i+\sum_{\forall k\in\mathcal{J}\backslash j}\frac{p_i^k}{q^k}\big)}{\sum_{\forall k\in\mathcal{J}\backslash j}S_i^k}} \\
        &=\sqrt{\frac{q^jc^jS_i^j\big(\chi^2R_i+\chi^2\sum_{\forall k\in\mathcal{J}\backslash j}\frac{p_i^k}{q^k}\big)}{\sum_{\forall k\in\mathcal{J}\backslash j}S_i^k}} \\
        &>\sqrt{\frac{q^jc^jS_i^j\big(R_i+\chi\sum_{\forall k\in\mathcal{J}\backslash j}\frac{ p_i^k}{q^k}\big)}{\sum_{\forall k\in\mathcal{J}\backslash j}S_i^k}}=\phi_i^j(\chi\boldsymbol{p}).
    \end{split}
\end{equation}

Based on the above analysis, it has been demonstrated that the best response function of RSU $j$ in this case has a unique optimal point.
\end{proof}

 {In the case that this Stackelberg game has a Stackelberg equilibrium, each RSU independently determines its optimal strategy, and no RSU can improve its revenue by unilaterally altering its chosen strategy. Once all RSUs have declared their pricing strategies, the UAVs proceed to determine their optimal bandwidth demands to maximize their utilities.} 

 {However, obtaining the private information of UAVs in advance poses practical challenges for RSUs in real-world scenarios, primarily due to legitimate concerns regarding privacy protection. Therefore, the development of a privacy-preserving algorithm becomes imperative for the RSUs to derive their optimal policies under such circumstances.}

\section{Tiny Multi-agent Deep Reinforcement Learning Algorithm }\label{Algorithm}

 {In complex decision-making scenarios, advanced AI techniques (e.g., DRL) offer promising solutions for devising incentive mechanisms under privacy concerns\cite{9492053,9430722,8984310}. In the preceding section, a multi-leader multi-follower Stackelberg game is formulated between RSUs and UAVs, demonstrating the existence and uniqueness of the Stackelberg equilibrium \cite{du}. In this section, we model the Stackelberg game between RSUs and UAVs as an MDP \cite{8758205, MADDPG}. To tackle the challenge posed by incomplete information, we propose a Tiny MADRL algorithm that aims to achieve the Stackelberg equilibrium by exploring the optimal solutions for both RSUs and UAVs within the Stackelberg model. 
The Tiny MADRL algorithm empowers individual RSUs as agents to rapidly converge to near-optimal decision-making capabilities, all within a decentralized and privacy-conscious framework \cite{9492053}. Contrasting traditional approaches in DRL that focus on estimating fixed policies or single-step models, the proposed method leverages Markov properties to effectively decompose the problem.}

\subsection{MDP for the Stackelberg Game between RSUs and UAVs}
 {An environment that follows an MDP is needed to train DRL agents \cite{9235308}, which is formulated by conceptualizing the dynamic relationship between the RSUs and UAVs as a Stackelberg game. This strategic interaction is then formally cast as a multi-agent MDP. Within this framework, each RSU functions as an independent agent, navigating its decision-making within the MDP context. The environment is populated by the $I$ UAVs, collectively contributing to the complexity and richness of the learning environment. Specifically, let $\mathcal{M}^j=\{\mathcal{S}^j,\mathcal{A}^j,\mathcal{P}^j,\mathcal{R}^j,\gamma^j\},j\in{\mathcal{J}}$ represent a MDP \cite{lifirst}, where $\mathcal{S}^j$, $\mathcal{A}^j$, $\mathcal{P}^j$, $\mathcal{R}^j$ and $\gamma^j$ represent state space, action space, state transition probability, reward function, and the discounted factor for RSU $j$, respectively \cite{8758205, lifirst}. }

In each time step $t$, where $t\in{\mathcal{T}=\{0,\ldots,t,\ldots, T\}}$, RSU $j$ interacts with the environment to determine its current state, denoted as $\boldsymbol{s}^j(t)$. 
Throughout the training process, the DRL agent, representing RSU $j$, engages in interactions with the environment. At each time step, when RSU $j$ executes the action $\boldsymbol{p}^j(t)$ based on the current state $\boldsymbol{s}^j(t)$, the environment provides an immediate reward $r^j(t)$ \cite{9235308}. RSU $j$ takes on the role of a game leader, responsible for selecting the action, i.e., the pricing policy $\boldsymbol{p}^j(t)$. Subsequently, a UAV acts as a follower, making an optimal strategy decision based on Eq. (\ref{b*}). Following this, the environment provides a reward $r^j(t)$ to RSU $j$, considering the strategies chosen by all UAVs. The system contains a finite-buffer queue, denoted as $\mathcal{D}$, to store historical operation data, and the capacity of the finite-buffer queue is defined as $L$. Information relevant to RSUs is extracted from this buffer to generate new states, and this process initiates the next time step \cite{9492053}.

\subsubsection{State Space}
We denote $\mathcal{S}^j\triangleq\{\boldsymbol{s}^j\}$ as the state space of RSU $j$. The state of RSU $j$ includes its current status and the cumulative history of all past policies enacted by RSU $j$ and UAV towards RSU $j$ \cite{zhan2020learning}. 
 {The price vector of RSU $j$ and the bandwidth demand vector of all UAVs regarding RSU $j$ in the time step $t$ are denoted as $\boldsymbol{p}^j(t)$ and $\boldsymbol{b}^j(t)$, respectively. The state of RSU $j$ is determined by the historical records of the previous $L$ games involving RSU $j$ and all UAVs. Consequently, at the time step $t$, the state of RSU $j$ is represented as $\boldsymbol{s}^j(t)=\{\boldsymbol{p}^j(t-L),\boldsymbol{b}^j(t-L),\ldots,\boldsymbol{p}^j(t-1),\boldsymbol{b}^j(t-1)\}$. }

\subsubsection{Action Space}
 {We denote $\mathcal{A}^j\triangleq\{\boldsymbol{p}^j\}$ as the action space of RSU $j$. At each time step $t$, RSU $j$ decides its action $\boldsymbol{p}^j(t)$. This decision-making process relies on the information encapsulated in the observed state $\boldsymbol{s}^j(t)$. }


\subsubsection{Reward Function}
We denote $\mathcal{R}^j\triangleq\{r^j\}$ as the reward function of RSU $j$. After the state transition, RSU $j$ acquires a reward that relies on the current state $\boldsymbol{s}^j(t)$ and the corresponding action $\boldsymbol{p}^j(t)$ \cite{8758205}.
The reward function of RSU $j$ is defined as the utility function of RSU $j$ that we construct in the Stackelberg game. At time step $t$, the reward function for RSU $j$ is represented as $r^j(t) = V_j\left(\boldsymbol{p}^j(t), \boldsymbol{p}^{-j}(t), \boldsymbol{b}^j(t)\right)$. The average reward across all RSUs is denoted as $\frac{\sum_{j\in\mathcal{J}} V_j(\cdot)}{J}$.

\begin{figure*}[ht]
\centering
\includegraphics[width=0.95\textwidth]{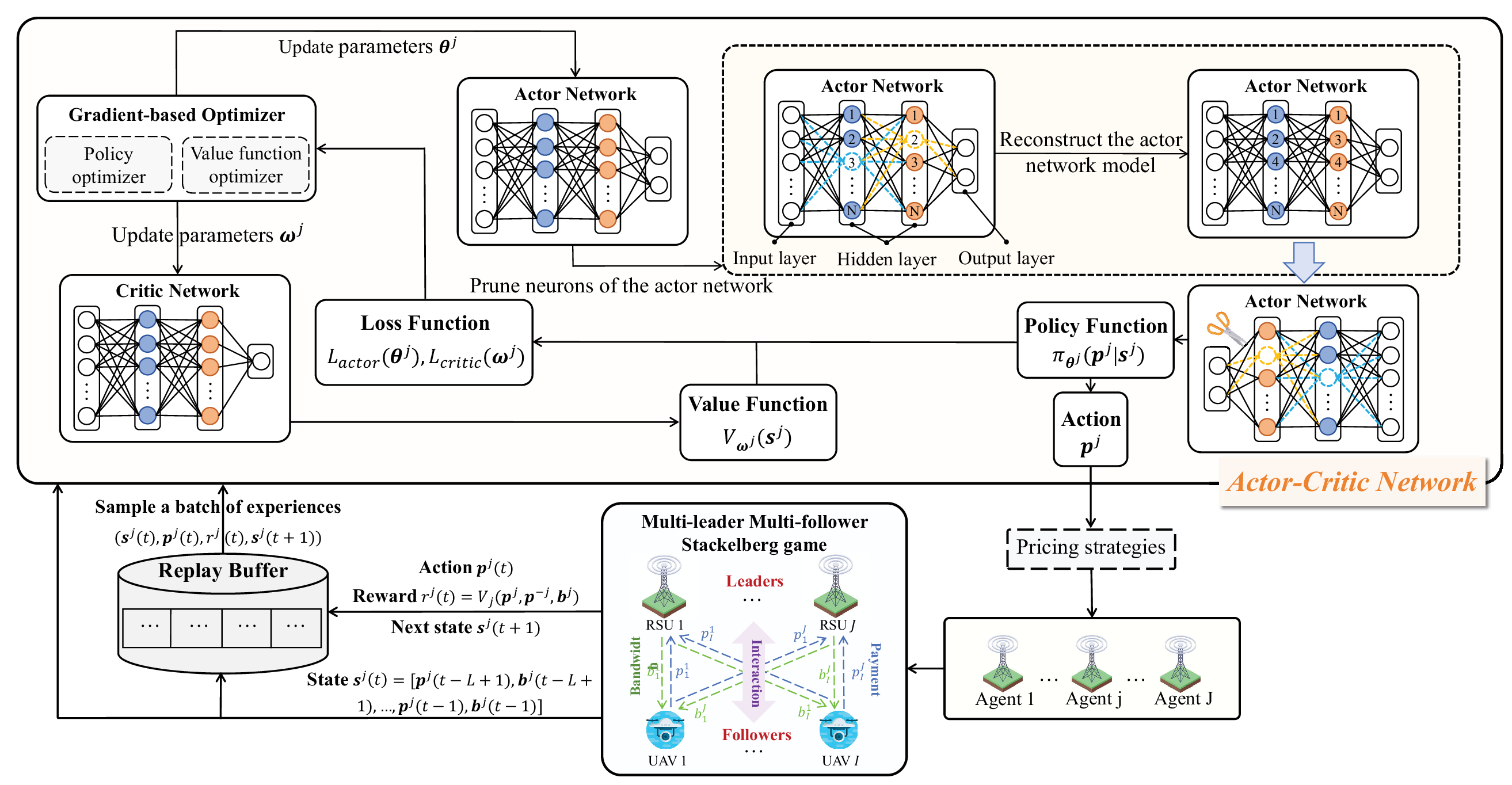}
\caption{ {The Tiny MADRL framework of the dynamic structured pruning algorithm for the Tiny MADRL network. Inactive neurons are represented by dashed circles, while deleted weights are denoted by dotted lines.}}
\label{MAPDRL}
\end{figure*} 

In the actor-critic network setup, there are two integral components, i.e., the actor network and the critic network \cite{cheng}. PPO is a DRL algorithm based on policy gradients. It enhances stability and convergence by employing proximal optimization techniques on the policy, ensuring more reliable and efficient learning processes. We denote the actor-critic network in the Tiny MADRL framework as $(\boldsymbol \theta^j, \boldsymbol \omega^j), j\in\mathcal{J}$. In the PPO algorithm, for each RSU $j$, on the one hand, the actor network operates as a policy function  ${\pi_{\boldsymbol \theta^j}}(\boldsymbol{p}^j|\boldsymbol{s}^j)$ with parameters $\boldsymbol \theta^j$. This function is instrumental in generating the action of RSU $j$, i.e., the pricing strategy $\boldsymbol{p}^j$, and facilitating interactions with the environment. The critic network, on the other hand, represented by the value function $V_{\boldsymbol \omega^j}(\boldsymbol{s}^j)$ parameterized by $\boldsymbol \omega^j$, evaluates the performance of the actor (i.e., the agent) and guides the actions of actors in subsequent phases.  {Specifically, the primary goal of the critic network is the minimization of the Temporal Difference (TD) error, expressed as
\begin{equation}
    d^j=r^j(t)+\gamma^j V_{\boldsymbol \omega^j}\big(\boldsymbol{s}^j(t+1)\big)-V_{\boldsymbol \omega^j}\big(\boldsymbol{s}^j(t)\big),
\end{equation}
where $V_{\boldsymbol \omega^j}\big(\boldsymbol{s}^j(t)\big)$ and $V_{\boldsymbol \omega^j}\big(\boldsymbol{s}^j(t+1)\big)$ correspond to the value functions associated with the current state $\boldsymbol{s}^j(t)$ and the subsequent state $\boldsymbol{s}^j(t+1)$, respectively.
Therefore, the loss function for the critic network is obtained through the minimization of the expected value of the squared Temporal Difference (TD) value, which can be represented as \cite{li2022compact}
\begin{equation}
\begin{split}
    \mathop{\min}\limits_{\boldsymbol \omega^j}L_{critic}(\boldsymbol \omega^j) = \mathop{\min}\limits_{\boldsymbol \omega^j}\mathbb{E}\big[\big(r^j(t)+\gamma^j V_{\boldsymbol \omega^j}\big(\boldsymbol{s}^j(t+1)\big) \\
    -V_{\boldsymbol \omega^j}\big(\boldsymbol{s}^j(t)\big)\big)^2\big].
\end{split}
\end{equation}
In addition, the aim of the actor network is defined as
\begin{equation}
\begin{split}
    \mathop{\max}\limits_{\boldsymbol \theta^j}L_{actor}(\boldsymbol \theta^j)=\mathop{\max}\limits_{\boldsymbol \theta^j}\mathbb{E}\big[\min\big(f^j(\boldsymbol \theta^j)\hat A_{{\pi_{\boldsymbol \theta^j}}}(\boldsymbol{s}^j,\boldsymbol{p}^j), \\
    \eta\big(f^j(\boldsymbol \theta^j)\big)\hat A_{{\pi_{\boldsymbol \theta^j}}}(\boldsymbol{s}^j,\boldsymbol{p}^j)\big)\big],
\end{split}
\end{equation}
where $f^j(\boldsymbol \theta^j)=\frac{{\pi_{\boldsymbol \theta^j}}(\boldsymbol{p}^j|\boldsymbol{s}^j)}{\pi_{\hat{\boldsymbol\theta}^j}(\boldsymbol{p}^j|\boldsymbol{s}^j)}$, $\hat{\boldsymbol \theta}^j$ is the parameter of the strategy used for sampling $\boldsymbol{p}^j$, and $\pi_{\hat{\boldsymbol\theta}^j}(\boldsymbol{p}^j|\boldsymbol{s}^j)$ denotes the policy employed for importance sampling \cite{9492053}. $\eta(f^j)$ is a piece-wise function with intervals given as follows \cite{9492053}:
\begin{equation}
    \eta(f^j)=
    \begin{cases}
        1+\kappa,&f^j>1+\kappa, \\
        f^j, &1-\kappa\le{f^j}\le{1+\kappa}, \\
        1-\kappa, &f^j<1-\kappa, 
    \end{cases}
\end{equation}}
 {where $\kappa$ is an adjustable hyper-parameter, and $\hat A_{{\pi_{\boldsymbol \theta^j}}}(\boldsymbol{s}^j,\boldsymbol{p}^j)$ denotes the estimator for the advantage function $A_{{\pi_{\boldsymbol \theta^j}}}(\boldsymbol{s}^j,\boldsymbol{p}^j)$, which is calculated as
\begin{equation}
    \hat A_{{\pi_{\boldsymbol \theta^j}}}\big(\boldsymbol{s}^j(t),\boldsymbol{p}^j(t)\big)=-V_{{\pi_{\boldsymbol \theta^j}}}\big(\boldsymbol{s}^j(t)\big)+\sum_{l=0}^{\infty}(\gamma^j)^lr(t+l).
\end{equation}}

The Tiny MADRL framework utilizes classic neural networks, specifically DNNs, to parameterize the actor-critic network \cite{ji2021clnet}. Considering the network structure, actor and critic networks typically consist of fully connected layers comprising an input layer, multiple hidden layers, and an output layer housing numerous parameters like neurons and weights \cite{ahmed2019deep}. The following analysis is for RSU $j$, $j\in\mathcal{J}$. We consider an actor network with $H$ layer, and denote the weights in the $h$-th fully connected layer as $\boldsymbol{\theta}^{j(h)}$, where $h\in\{1,2,\ldots,H\}$. By setting the bias of the DNN to $0$ and inputting the state $\boldsymbol{s}(t)$ at time step $t$ into the first layer, the output of the first layer is calculated as
\begin{equation}
    \boldsymbol o^{j(1)} = \sigma^{j(1)}\big(\boldsymbol{\theta}^{j(h)}\boldsymbol{s}^j(t)\big),
\end{equation}
where $\sigma^{j(1)}$ represents the nonlinear response of the first layer, usually set to the ReLU function. The output of each layer in the network is fed as the input to the subsequent layer \cite{ahmed2019deep}. Therefore, the output of the $h$-th layer is expressed as 
\begin{equation}
    \boldsymbol o^{j(h)} = \sigma^{j(h)}\big(\boldsymbol\theta^{j(h)}\boldsymbol o^{j(h-1)}\big).
\end{equation}
Finally, the actor network outputs the action in time step $t$, i.e., the price strategy $\boldsymbol{p}^j(t)$, which is shown as
\begin{equation}
    \boldsymbol{p}^j(t)=\sigma^{j(H)}\big(\boldsymbol\theta^{j(H)}\boldsymbol o^{j(H-1)}\big).
\end{equation}

\subsection{ The Parameters of Actor-Critic Network Update}
 {Furthermore, we leverage the pruning techniques into DRL to eliminate neurons and weights that do not significantly contribute to the performance of the actor network. A binary mask $\boldsymbol m^{(h)}$ is introduced for each neuron to indicate the pruning status of that neuron, i.e., $1$ for no pruning and $0$ for pruning \cite{9727767}.} Then, the action output from actor network is expressed as
\begin{equation}
    \boldsymbol{p}^j(t)=\sigma^{j(H)}\big(\boldsymbol \theta^{(H)}\boldsymbol o^{j(H-1)}\odot \boldsymbol m^{j(H)}\big),
\end{equation}
where $\odot$ denotes the Hadamard product. Then, the loss function of actor network is
\begin{equation}
\begin{split}
    L_{actor}(\boldsymbol \theta^j)=\mathbb E \big[\min\big(f^j(\boldsymbol \theta^j,\boldsymbol m)\hat A_{{\pi_{\boldsymbol \theta^j}}}(\boldsymbol{s}^j,\boldsymbol{p}^j), \\
    \eta\big(f^j(\boldsymbol \theta^j,\boldsymbol m)\big)\hat A_{{\pi_{\boldsymbol \theta^j}}}(\boldsymbol{s}^j,\boldsymbol{p}^j)\big)\big].
\end{split}
\end{equation}

Upon the accumulation of $D$ records in the replay buffer, the actor and critic networks undergo updates. Specifically, RSU $j$ updates the parameters of the actor network using the gradient ascent method, which can be expressed as
\begin{equation}
    {\boldsymbol{\theta}^{j(h)}}'=\boldsymbol{\theta}^{j(h)}-l^j_1\frac{\partial L_{actor}(\boldsymbol \theta^j)}{\partial \big(\boldsymbol o^{j(h)}\odot \boldsymbol m^{j(h)}\big)}\cdot\frac{\partial \big(\boldsymbol o^{j(h)}\odot \boldsymbol m^{j(h)}\big)}{\partial \boldsymbol{\theta}^{j(h)}},
    \label{theta}
\end{equation}
where $l^j_1$ denotes the learning rate employed for training the actor network, and ${\boldsymbol{\theta}^{j(h)}}'$ is the updated parameter of the actor network. Moreover, the parameters of the critic network are updated through the gradient descent method as follows \cite{8758205}:
\begin{equation}
    {\boldsymbol{\omega}^{j(h)}}'=\boldsymbol{\omega}^{j(h)}-l^j_2\frac{\partial L_{critic}(\boldsymbol \omega^j)}{\partial \boldsymbol{\omega}^{j(h)}},
    \label{omega}
\end{equation}
where $l^j_2$ denotes the learning rate utilized in the training process of the critic network, while ${\boldsymbol \omega^{j(h)}}'$ signifies the updated parameter of the critic network.

\subsection{Pruning Algorithm}
 {Due to the challenges associated with using unstructured pruning techniques for accelerating DRL training, which often results in irregular network structures, we opted for using structured pruning methods. Structured pruning is an approach to reduce model complexity by strategically eliminating redundant neurons or connections. Dynamic structured pruning of non-essential neurons encompasses two key steps: (i) determining the pruning threshold, and (ii) updating the binary mask used for pruning \cite{Liu}. The pruning threshold plays a crucial role in identifying the parameters or connections that are essential to retain, while identifying and eliminating those that are unnecessary during the pruning process \cite{zhaodynamic}.
We employ a dynamic pruning threshold, which is defined as \cite{livne2020pops}
\begin{equation}
    \psi(t) = \sum_{n=1}^{N}\sum_{h=1}^{H}\phi_n^{(h)}\cdot w(t),
    \label{psi}
\end{equation}
\begin{equation}
    w(t) = \check w+(\hat w-\check w)\bigg(1-\frac{t-t_0}{M\nabla t}\bigg)^3,
\end{equation}
where $\phi_n^{(h)}$ and $M$ represent the neuronal importance of the $n$-th neuron of layer $h$ and the total number of pruning steps, respectively. $t_0$ and $\nabla$ represent the starting epoch of gradual pruning and the pruning frequency, respectively. $w(t)$, $\hat w$, and $\check w$ correspond to the current sparsity in epoch $t$, initial sparsity, and target sparsity \cite{livne2020pops}, respectively. This dynamic pruning approach is designed to adaptively enhance the sparsity of the model as iterations progress, offering a more refined and effective method for structured pruning. Neurons are sorted based on their importance, from the least to the most significant. Neurons that are ranked below the established threshold are then pruned, contributing to the overall sparsity of the model.} The mask of the $n$-th neuron of layer $h$ is updated as
\begin{equation}
    m_n^{(h)}=
    \begin{cases}
        1, &\mathrm{if} \:\mathrm{abs} \left[m_n^{(h)}, \theta_n^{(h)}\right]\ge\psi , \\
        0, &\mathrm{otherwise}.
    \end{cases}
    \label{m}
\end{equation}

\begin{algorithm}[t]
\label{MAPDRL_algorithm}
\DontPrintSemicolon
\SetAlgoLined
\KwIn {State $\boldsymbol{s}^j$.}
\KwOut {Compact Tiny MADRL model $(\boldsymbol{\theta}^j, \boldsymbol{\omega}^j)^{(H)}, j\in \mathcal{J}$.}
Initialize Tiny MADRL model, training episodes $T$, reward $r^j$, and binary mask $\boldsymbol{m}^j$. \;
\For{\rm{time step} $t=1$ to $T$}
{
    Agents interact with UAVs. \;
    Compute neuron importance $\phi_n^{(h)}$. \;
    Update actor network parameter $\boldsymbol{\theta}^{j(h)}$ and critic network parameter $\boldsymbol{\omega}^{j(h)}$ by Eqs. (\ref{theta}) and (\ref{omega}), respectively.\;
    Compute dynamic pruning threshold $\psi^{(h)}$ by Eq. (\ref{psi}).\;
    Update binary mask $\boldsymbol m^{j(h)}$ by Eq. (\ref{m}).\;
    \If{$\phi_n^{(h)} < \psi$} {
        Remove $n$-th neuron in $h$-th layer and associated parameters $\boldsymbol{\theta}$ from actor network.\;
    }
}
Reconstruct the compact Tiny MADRL model.
\caption{Dynamic Structured Pruning in the Tiny MADRL Framework.}
\end{algorithm}

 {Figure \ref{MAPDRL} shows the proposed Tiny MADRL framework in this paper. The framework of the proposed Tiny MADRL algorithm is demonstrated, where the DRL model is first trained, and then a dynamic structural pruning method is employed in the actor-critic network, eliminating groups of the least important neurons. For instance, the importance of the $3$-th neuron in the $2$-th layer $\phi_3^{(2)}$ is the smallest value in layer $2$, so it is pruned out. The process of dynamic structured pruning is shown in \textbf{ Algorithm \ref{MAPDRL_algorithm}}. In the Tiny MADRL model, we adopt a fully-connected DNN architecture for the actor network, consisting of $H$ layers. \textbf{ Algorithm \ref{MAPDRL_algorithm}} consists of a two-step process as presented in Fig. \ref{MAPDRL}, i.e., training the DRL model initially (lines $4$ to $5$), followed by the removal of unimportant neurons using the dynamic pruning threshold (lines $6$ to $10$). 
The complexity of \textbf{Algorithm \ref{MAPDRL_algorithm}} over $T$ episodes is $\mathcal O\big(TV\big)+\mathcal O\big(T\sum_{h=1}^{H-1}u^{(h)}\big)$. It is characterized by two main components, i.e., $\mathcal O(TV)$, reflecting the computation cost associated with each episode and the size of the state vector $V$, and $\mathcal O\big(T\sum_{h=1}^{H-1}u^{(h)}\big)$, representing the cumulative computational complexity over $T$ episodes with respect to the number of neurons $u^{(h)}$ in each hidden layer $h$ up to the penultimate layer.}

\section{Numerical Results}\label{Results}
This section presents numerical results to empirically demonstrate the efficacy of the proposed scheme. Similar to \cite{8792382,zhongblockchain}, the critical parameters of the experiment are listed in Table \ref{parameters}. 

\begin{table}[t]
  \begin{center}
    \caption{Key Parameters in the Simulation.}
    \label{parameters}
    \begin{tabular}{l|c} 
    \toprule 
      \textbf{Parameters} & \textbf{Values}\\
      \hline
      Additive white Gaussian noise $\sigma^{j2}$ from RSU $j$ & $[-116, -112] \rm{dBm}$ \\ 
      transmitting a rendered screen to UAVs   \\
      Channel gain $g^j$ from RSU $j$ transmitting a & $[-25, -22] \rm{dB}$   \\
      rendered screen to UAVs\\
      Transmit power $m^j$ from RSU $j$ transmitting a & $[20, 25] \rm{dBm}$  \\
      rendered screen to UAVs\\
      The brightness similarity $l_i^j$ between the rendered & $[0,1]$ \\
      image in RSU $j$ and the received image by UMUs \\
      The contrast similarity $c_i^j$ between the rendered & $[0,1]$ \\
      image in RSU $j$ and the received image by UMUs \\
      The structure similarity $s_i^j$ between the rendered & $[0,1]$ \\
      image in RSU $j$ and the received image by UMUs \\
      Threshold of minimum SSIM $SSIM_i^{th}$ of UAV $i$ & $[0.5, 0.55]$ \\
      Parameter $\delta_i$ of UAV $i$ & $[10, 20]$ \\
      The bandwidth cost $c^j$ & $[1,4]$ \\
      The upper bandwidth price $\bar p^j$ & $[5,35]$ \\
      \bottomrule
    \end{tabular}
  \end{center}
\end{table}

\begin{figure}[t]
\centering
\includegraphics[width=0.45\textwidth]{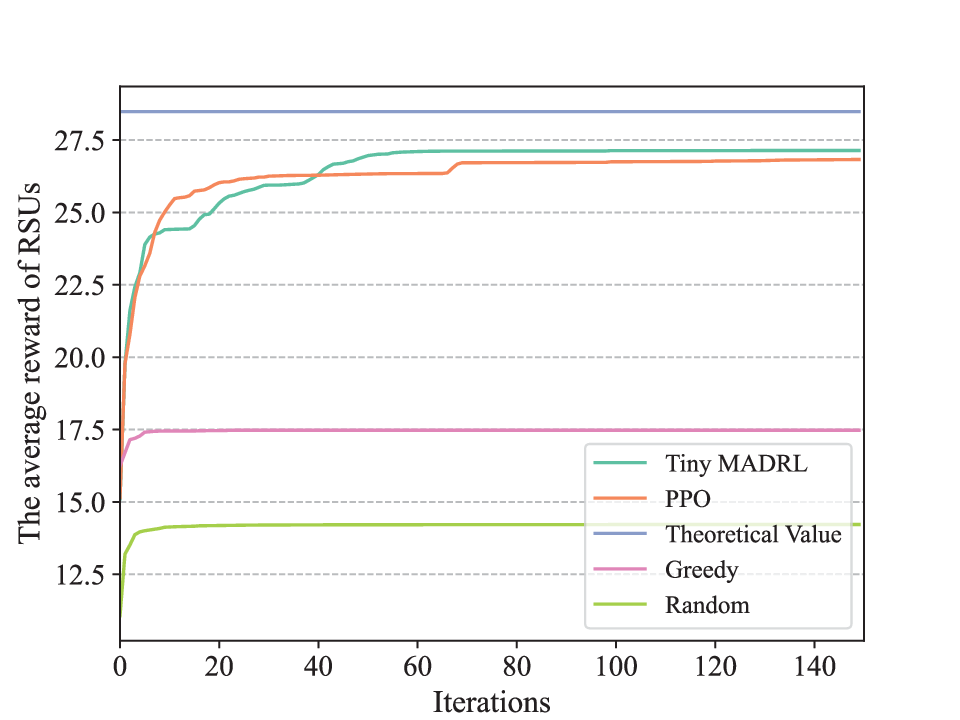}
\caption{Convergence of the Tiny MADRL algorithm.}
\label{convergence}
\end{figure}

Firstly, we demonstrate the convergence of the Tiny MADRL algorithm proposed in this paper. As iteration increases, Fig. \ref{convergence} illustrates the alteration in the average reward of RSUs through different algorithms. In Fig. \ref{convergence}, we consider $5$ RSUs and $15$ UAVs in the system, and we visualize the convergence behavior of the Tiny MADRL algorithm introduced.  {We find that the Tiny MADRL algorithm converges faster than the traditional PPO algorithm. Moreover, the final results achieved by the Tiny MADRL algorithm are closer to the theoretical value than the traditional PPO, greedy, and random algorithms.} Note that the theoretical values are obtained from the proposed equations in Section \ref{Stackelberg}.

\begin{figure}[t]
\centering
\subfigure[The pricing strategies of RSUs.]{
\centering
\includegraphics[width=0.45\textwidth]{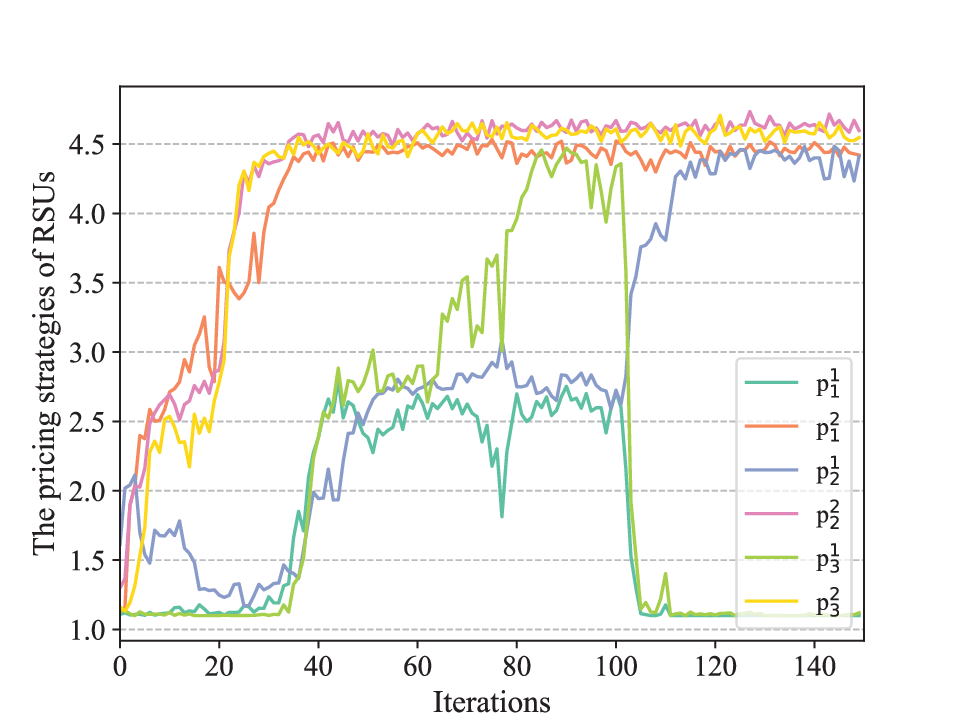}
\label{price}
}
\subfigure[The bandwidth demands of UAVs.]{
\centering
\includegraphics[width=0.45\textwidth]{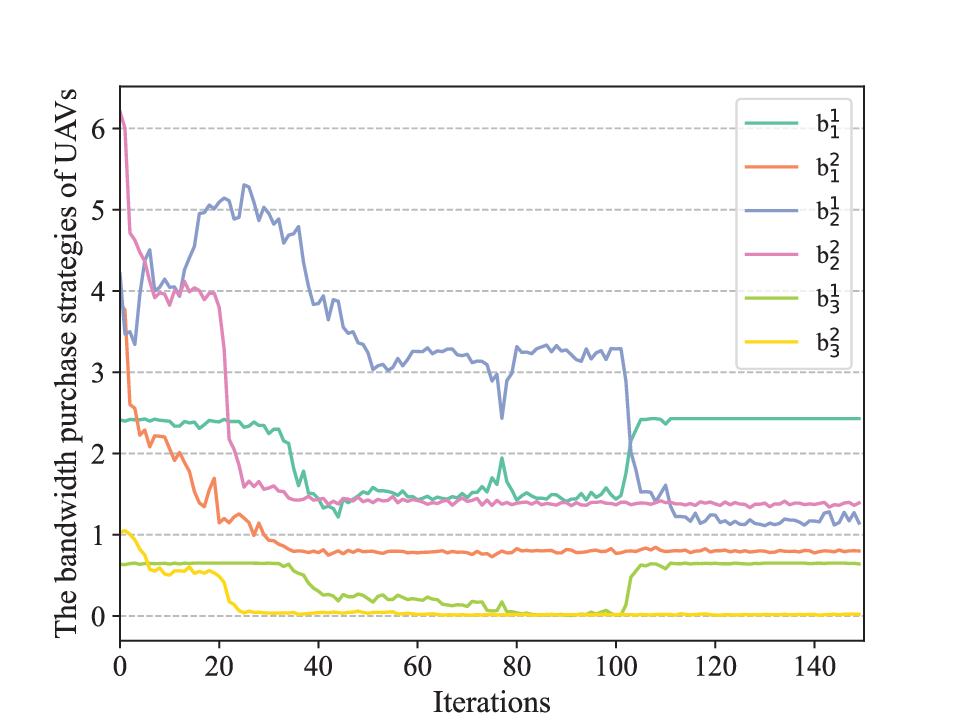}
\label{bandwidth}
}
\caption{The strategies of RSUs and UAVs obtained by the Tiny MADRL algorithm.}
\label{strategy}
\end{figure}

Figure \ref{strategy} shows the dynamic evolution of strategies employed by RSUs and UAVs as the number of iterations increases. Figure \ref{price} shows the variation of strategies of RSUs over multiple iterations. The experimental setup assumes the presence of $2$ RSUs and $3$ UAVs, resulting in the $6$ RSU strategies, each represented by a curve in Fig. \ref{price}. The convergence of strategies of all RSUs is evident in Fig. \ref{price}. Notably, the strategy of each RSU for each UAV is crafted to maximize its utility. 
Moreover, the graphical representation in Fig. \ref{bandwidth} elucidates the dynamic evolution of strategies of UAVs concerning the iteration count. From Fig. \ref{bandwidth}, we can see that the strategies of the UAVs all converge around the $110$-th iteration. 
A comparison between Fig. \ref{price} and Fig. \ref{bandwidth} reveals that the bandwidth price set by the RSU and the bandwidth demand of the UAV exhibit opposing patterns. For instance, the bandwidth price set by RSU $1$ for UAV $1$ initially increases and then decreases, while the bandwidth demand of UAV $1$ from RSU $1$ decreases and then increases. In summary, both the strategies of RSUs and UAVs can reach convergence through the Tiny MADRL algorithm. For the same RSU and the same UAV, the higher the bandwidth price set by the RSU, the less bandwidth is demanded for that UAV.

\begin{figure}[t]
\centering
\subfigure[Relation between the average reward of RSUs and the bandwidth cost $c$.]{
\centering
\includegraphics[width=0.45\textwidth]{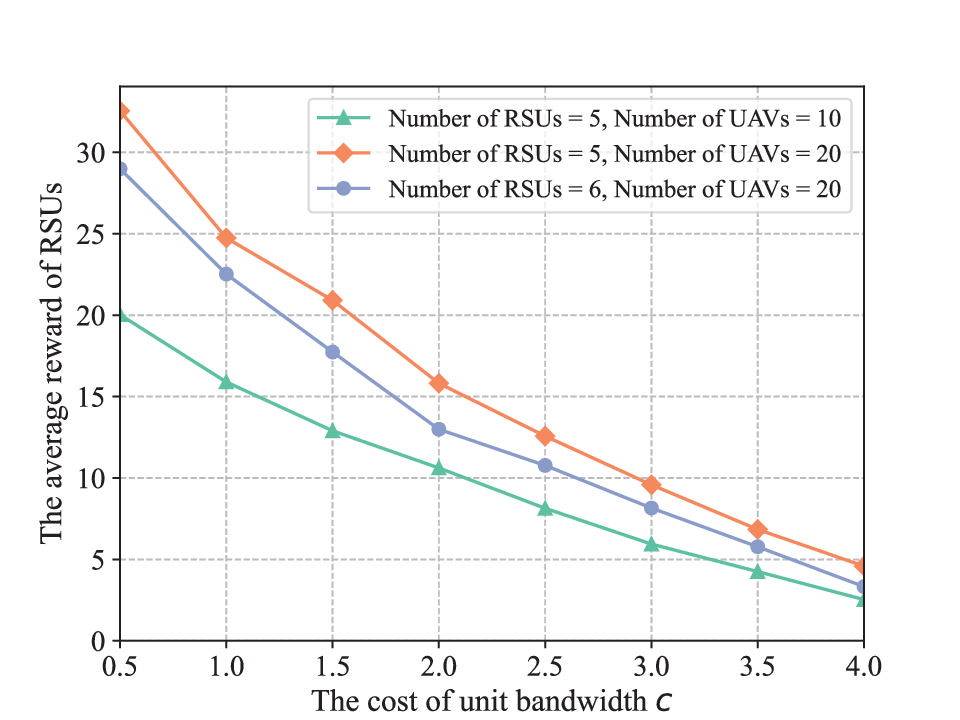}
\label{cost}
}
\subfigure[Relation between the average reward of RSUs and the upper bandwidth selling price $\bar p$.]{
\centering
\includegraphics[width=0.45\textwidth]{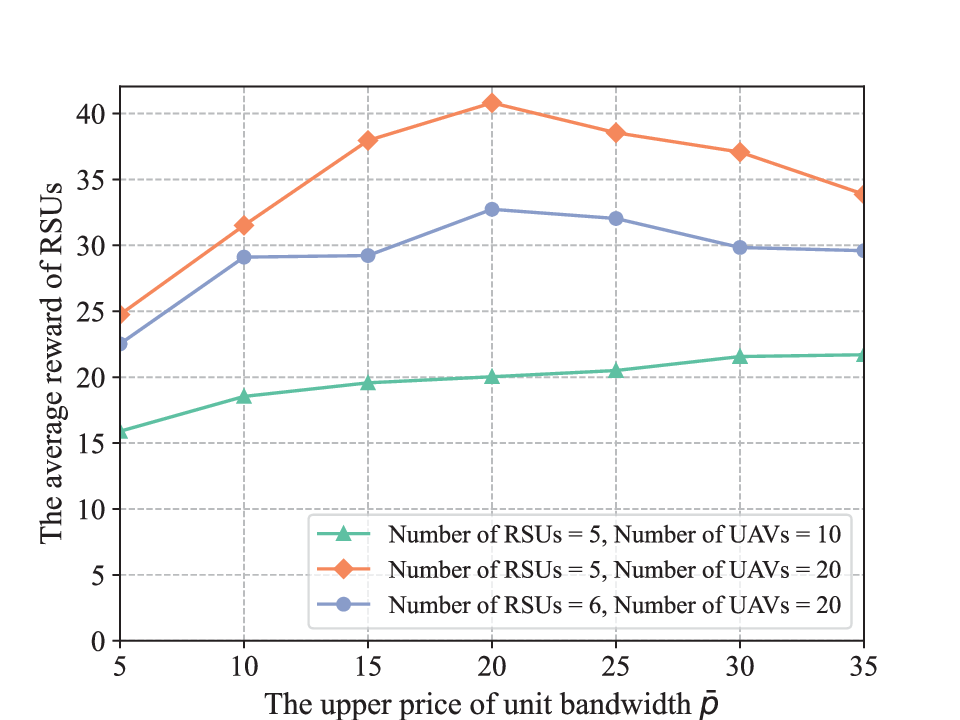}
\label{p_max}
}
\caption{Relation between the average reward of RSUs and the parameters.}
\label{c_p}
\end{figure}

 {Figure \ref{c_p} illustrates the impact of increasing bandwidth cost, denoted as $c$, and the upper bandwidth price, denoted as $\bar p$,  on the average reward for RSUs. As shown in Fig. \ref{cost}, the average reward declines with rising bandwidth costs. This is due to a fixed upper bandwidth price, the price strategies that RSUs can choose are relatively fixed, and costs continue to rise, resulting in a decreasing gap between the bandwidth selling price of RSUs and the bandwidth cost. Consequently, the utility of RSUs decreases, as defined by Eq. (\ref{V}). 
Therefore, an increase in the cost of bandwidth results in a decrease in average reward. As the quantity of UAVs expands under the condition of a constant number of RSUs, the average reward of RSUs also rises due to the growing demand for bandwidth. It is important to note that our experiments assume unlimited resources from RSUs. However, with a constant number of UAVs, increasing the number of RSUs leads to a decrease in average reward. This is likely due to the fixed demand for bandwidth from UAVs, and an increase in RSUs diminishes the average reward per RSU. In Fig. \ref{p_max}, we can see that in some cases the average reward of RSUs grows to a maximum value and then decreases. The average reward of RSUs initially increases with a higher $\bar p$, as RSUs can set a higher bandwidth selling price. However, beyond a certain point, an elevated $\bar p$ may result in a decrease in the average reward due to a potential reduction in the bandwidth demands of UAVs. The observed increase in average reward with a growing number of UAVs, while the number of RSUs remains constant, is attributed to a larger demand for bandwidth. Conversely, an increase in the number of RSUs, with a constant number of UAVs, leads to a decrease in average reward. }

\begin{figure}[t]
\centering
\subfigure[The average reward of RSUs for different numbers of RSUs under different algorithms.]{
\centering
\includegraphics[width=0.45\textwidth]{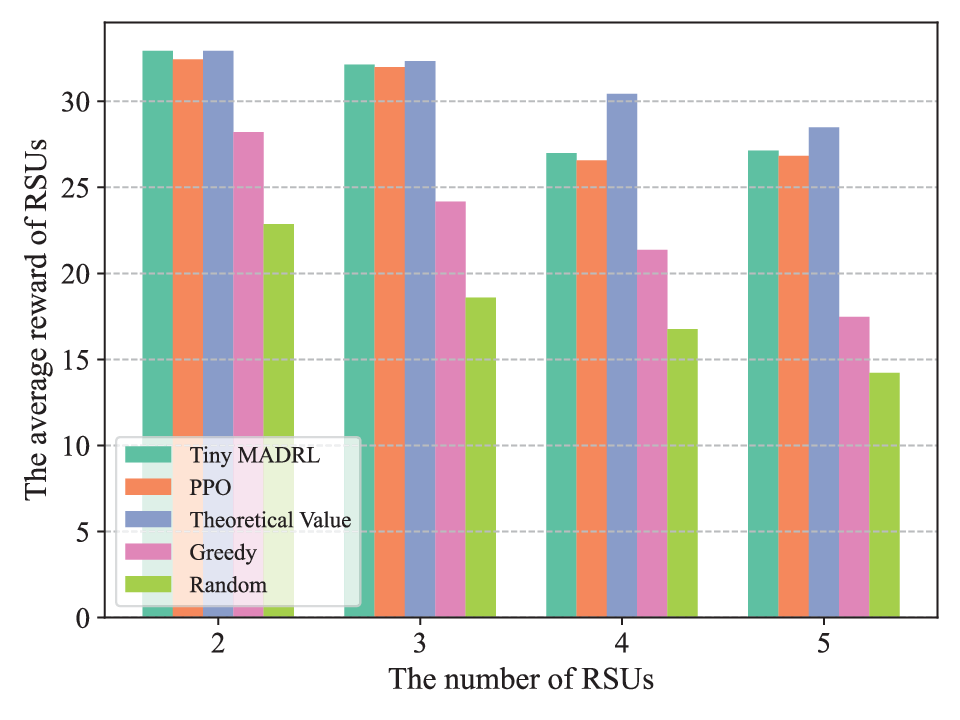}
\label{RSU_reward}
}
\subfigure[The average reward of RSUs for different numbers of UAVs under different algorithms.]{
\centering
\includegraphics[width=0.45\textwidth]{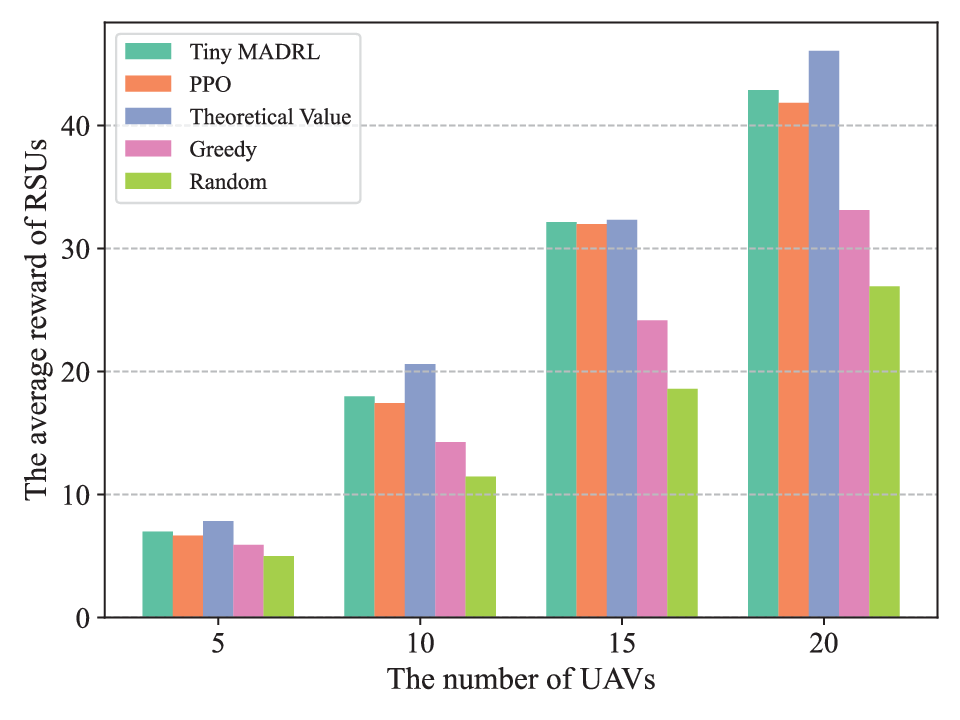}
\label{UAV_reward}
}
\caption{The average reward of RSUs for different numbers of UAVs/RSUs under different algorithms.}
\label{reward}
\end{figure}

Figure \ref{reward} shows the average reward of RSUs under different numbers of RSUs solved by different algorithms for the same number of UAVs. From Fig. \ref{reward}, the Tiny MADRL algorithm, introduced in this paper, demonstrates performance closest to theoretical values, indicating its superior efficacy. Our proposed Tiny MADRL algorithm can show better performance than the traditional PPO algorithm. 
Figure \ref{RSU_reward} illustrates the relationship between the number of RSUs and the average reward of RSUs with a fixed number of UAVs set to $15$.  {The average reward of all RSUs decreases with the increase of the number of the RSUs. This is because the same number of UAVs may request the same amount of resources, resulting in about the same revenue for RSUs, the more RSUs there are, the less average reward RSUs receive.} Figure \ref{UAV_reward}, configured with a fixed number of $3$ UAVs, reveals a positive correlation between the average reward of RSUs and the number of UAVs. The reason is that the increase in UAVs creates a surge in demand for bandwidth resources, empowering RSUs to attract more UAVs for bandwidth purchases. Simultaneously, RSUs can strategically elevate their bandwidth prices, maximizing their rewards.

\section{Conclusion}\label{Conclusion}
 {In this paper, we introduced the UAV metaverses and investigated the UT real-time migration in UAV metaverses. We focused on the scenario where UAVs offer seamless metaverse services to UMUs by transferring their UTs among the RSUs. To achieve efficient UT migration in UAV metaverses, we proposed a tiny learning-based game approach framework based on the pruning techniques. Specifically, we formulated a multi-leader multi-follower Stackelberg game between RSUs and UAVs, where resource-constrained UAVs act as followers and aim to optimize their bandwidth purchase from RSUs. In contrast, RSUs act as leaders and set the optimal bandwidth price provided to UAVs. In particular, we incorporated a novel metaverse immersion metric into the utility function of UAVs. To quickly find the Stackelberg equilibrium and obtain the optimal strategies for both RSUs and UAVs, we proposed a Tiny MADRL algorithm based on structural pruning techniques, which can reduce the size and complexity of the model by removing redundant neurons or weights. Finally, we conducted experiments to validate the effectiveness and reliability of our proposed scheme.}

 {In the future, our focus will shift toward constructing the utility function of UMUs and delving deeper into the modeling of immersion metrics, possibly innovating new modeling of metaverse immersion. Furthermore, we aim to explore the integration of state-of-the-art techniques with DRL methods to tackle the Stackelberg equilibrium point. For instance, our research will endeavor to explore and investigate the combination of the diffusion model and DRL, as well as the fusion of Quantum techniques with DRL. The explorations will shed light on innovative approaches to address the challenges in achieving the Stackelberg equilibrium point.}


\bibliographystyle{IEEEtran}

\bibliography{ref}

\end{document}